\documentclass{article}
\usepackage{iclr2027_conference,times}

\usepackage{amsmath,amsfonts,bm}

\def\eqref#1{equation~\ref{#1}}

\def\1{\bm{1}}

\DeclareMathAlphabet{\mathsfit}{\encodingdefault}{\sfdefault}{m}{sl}
\SetMathAlphabet{\mathsfit}{bold}{\encodingdefault}{\sfdefault}{bx}{n}

\usepackage{xcolor}
\usepackage[colorlinks=true,linkcolor=red!70!black,citecolor=blue!70!black,urlcolor=blue!70!black]{hyperref}
\usepackage{url}
\usepackage{booktabs}
\usepackage{graphicx}
\usepackage{amsmath,amssymb,amsthm}
\usepackage{multirow}
\usepackage{subcaption}
\usepackage{microtype}
\usepackage{algorithm}
\usepackage{algpseudocode}
\usepackage{enumitem}
\setlist{nosep, leftmargin=1.5em}

\iclrfinalcopy

\newtheorem{theorem}{Theorem}

\title{The Alignment Paradox: How Post-Training Amplifies Confident Hallucinations in Language Models}

\author{
  \normalfont
  \textbf{Qingjia Huang}$^{1,2,*}$, \quad
  \textbf{Yakai Li}$^{1,2,*}$, \quad
  \textbf{Jianguo Wu}$^{1,2}$, \quad
  \textbf{Qihang Zhou}$^{1,2}$, \\
  \textbf{Aimin Yu}$^{1,2}$, \quad
  \textbf{Xiaoqi Jia}$^{1,2}$, \quad
  \textbf{Luping Ma}$^{1,2}$, \quad
  \textbf{Weijuan Zhang}$^{1,2,\dagger}$ \\[6pt]
  $^1$Institute of Information Engineering, Chinese Academy of Sciences, Beijing, China \\
  $^2$School of Cyber Security, University of Chinese Academy of Sciences, Beijing, China
}

\begin{document}

\maketitle

\begingroup
\renewcommand\thefootnote{}
\footnotetext{$^*$Equal contribution. $^\dagger$Corresponding author: \texttt{zhangweijuan@iie.ac.cn}.}
\endgroup

\begin{abstract}
Large language models (LLMs) can produce factually incorrect answers with high confidence, undermining their reliability and limiting the effectiveness of uncertainty-based error detection. While prior research attributes confident hallucinations to factors such as missing knowledge in training data, reasoning errors, or stochastic decoding, we uncover that post-training alignment itself is a primary driver of these errors, a phenomenon we call the \textbf{Alignment Paradox}. Across five model families evaluated on factual benchmarks, unaligned base models produce few high-confidence errors on long-tail factual queries, whereas instruction-tuned models multiply high-confidence errors ($p \ge 0.95$) by more than an order of magnitude (10$\times$ to 35$\times$). Layer-wise probing with the Logit Lens reveals that this overconfidence emerges in late layers, where wrong-answer margins expand past 4.0 points after remaining near zero across early and intermediate layers. These findings motivate limiting margin growth during post-training. We implement this principle through an entropy-dependent margin bound in direct preference optimization (DPO). In multi-epoch experiments with Mistral-7B, the bounded objective reduces high-confidence errors by up to 35.3\% relative to standard DPO while maintaining performance on evaluated general reasoning benchmarks. These results show that bounded margins mitigate confident hallucinations during post-training.
\end{abstract}

\section{Introduction}
\label{sec:intro}

When large language models produce errors with low confidence ($p \approx 0.3\text{--}0.5$), uncertainty-based error detection and selective prediction \citep{geifman2017selective} can reliably flag or defer the query. A critical vulnerability arises when models produce incorrect assertions with high confidence ($p \ge 0.95$). Such confident hallucinations evade threshold-based abstention, bypass automated verification, and undermine decision systems where abstention is preferable to erroneous completions \citep{amodei2016concrete,kadavath2022language}.

While prior research attributes factual hallucinations to pre-training knowledge gaps, multi-step reasoning errors, or decoding noise \citep{kadavath2022language,dziri2023faith,chuang2024dola,orgad2025llms}, post-training alignment, comprising supervised fine-tuning (SFT; \citealt{ouyang2022training}) and preference optimization \citep{bai2022training,rafailov2023direct}, is widely assumed to improve truthfulness and safety \citep{touvron2023llama,dubey2024llama,qwen2024technical}. In this work, we uncover that post-training alignment itself is a primary driver of high-confidence errors, a phenomenon we term the \textbf{Alignment Paradox}\footnote{Code and evaluation data are accessible at \url{https://github.com/star5o/HCE}.}. Although alignment enhances conversational adherence, it degrades confidence reliability on factual queries, shifting incorrect predictions into extreme confidence regimes ($p \ge 0.95$).

\begin{figure}[t]
\centering
\includegraphics[width=0.88\textwidth]{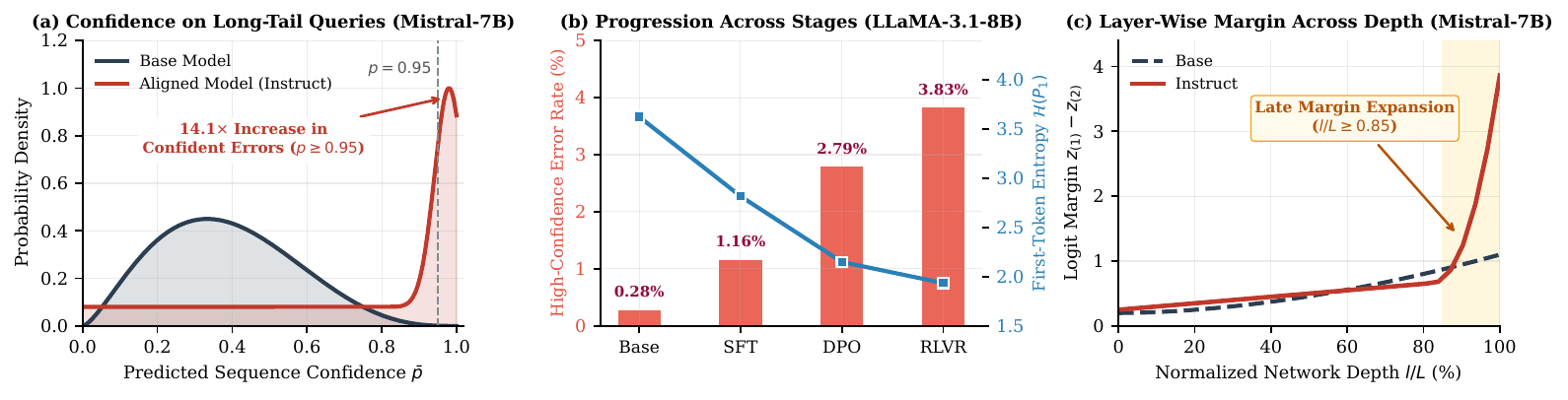}
\caption{\textbf{The Alignment Paradox.} Post-training alignment shifts output probabilities into high-confidence regimes and multiplies near-certain errors on factual queries. (a) Sequence confidence distributions on factual queries (\texttt{Mistral-7B}). (b) Progression of high-confidence error rates ($p \ge 0.95$) and first-token decision entropy across alignment stages in the \texttt{LLaMA-3.1-8B} lineage (Base $\to$ SFT $\to$ DPO $\to$ RLVR). (c) Layer-wise logit margin across normalized depth, showing logit margin expansion in late layers ($l/L \ge 0.85$, \texttt{Mistral-7B}).}
\label{fig:overview}
\vspace{-4pt}
\end{figure}

Empirical evaluation demonstrates this phenomenon across standard factual question-answering benchmarks. As illustrated in Figure~\ref{fig:overview}a, when an unaligned base model encounters an entity absent from its parametric memory, its output distribution remains diffuse across competing candidates, reflecting epistemic uncertainty. In contrast, post-trained checkpoints exhibit pronounced probability concentration, assigning probabilities exceeding $0.95$ even to false entities. Across five matched base and aligned model pairs spanning 7B to 14B parameters in the Mistral, Qwen, and LLaMA lineages over 22,267 factual queries, post-training alignment multiplies high-confidence errors ($p \ge 0.95$) by 10.0$\times$ to 35.4$\times$, worsening Expected Calibration Error (ECE; \citealt{guo2017calibration}) by up to 15.5-fold.

To determine whether this surge reflects conversational chat templates, we evaluate aligned models on bare completion prompts identical to those used for base models. Parameter updates account for the vast majority of the increase, as aligned checkpoints retain a 4.1$\times$ to 6.0$\times$ elevation in confident errors over base models even in the absence of instruction formatting. Furthermore, tracking sequential checkpoints across post-training reveals a two-stage progression, where supervised fine-tuning contracts decision entropy and preference optimization expands logit margins, driving incorrect outputs into extreme confidence intervals.

To investigate internal mechanisms, we probe intermediate representations using the Logit Lens \citep{belrose2023eliciting}. Our analysis shows that early commitment does not extend to factual errors, as decision distributions remain diffuse across early and intermediate representations. Overconfidence emerges strictly in late layers ($l/L \ge 0.85$), where wrong-answer margins expand past 4.0 points after hovering near zero through earlier layers, sharply suppressing decision entropy.

Inference-time interventions provide limited mitigation. Across five representative test-time defenses (defensive prompting, two-turn self-correction, sampling self-consistency, layer-contrastive decoding with DoLa \citep{chuang2024dola}, and entropy thresholding), we find that these heuristics eliminate at most 5\% to 8\% of high-confidence errors while preserving at least 97\% of correct completions. Furthermore, iterative self-correction degrades accuracy on initially correct completions. Because test-time heuristics cannot reverse parameter-level margin expansion, mitigating confident hallucinations requires regularizing the training objective itself.

We address this challenge by establishing the principle of bounded margin optimization for post-training alignment. Grounded in theoretical analysis showing that standard preference objectives drive logit margins toward infinity on ungrounded completions, we demonstrate that constraining margin expansion provides a general remedy against ungrounded overconfidence. We formulate Bounded Dynamic Margin (BDM) optimization, which enforces an adaptive, entropy-dependent ceiling on margin growth. During standard optimization on supported preferences, the margin bound remains inactive; under extended optimization on unsupported completions, it activates to prevent unbounded logit expansion. In multi-epoch evaluations across 22,267 factual queries, instantiating this principle via BDM-DPO reduces high-confidence errors by up to 35.3\% relative to standard DPO (eliminating 850 confident errors) while preserving general reasoning capabilities.

In summary, this work presents four primary contributions.
\begin{enumerate}
    \item \textbf{The Alignment Paradox.} We formalize the empirical phenomenon wherein post-training alignment multiplies high-confidence factual errors by 10.0$\times$ to 35.4$\times$ across five model families while degrading confidence reliability.
    \item \textbf{Causal Attribution.} Through prompt ablations and checkpoint tracking, we demonstrate that parameter updates rather than chat templates drive this effect, with preference objectives expanding logit margins on erroneous completions.
    \item \textbf{Representation Dynamics.} Using layer-wise Logit Lens probing, we show that early commitment does not hold for factual errors, demonstrating that overconfidence emerges strictly in late layers ($l/L \ge 0.85$).
    \item \textbf{Principled Mitigation.} We show that test-time defenses achieve at most an 8.3\% error reduction, and establish that constraining margin expansion provides a general training-time remedy. Instantiated as BDM-DPO, this bounded formulation reduces high-confidence errors by up to 35.3\% while preserving general reasoning performance.
\end{enumerate}

\section{Related Work}
\label{sec:related}

\paragraph{Factual Knowledge and Confidence Reliability.}
Evaluating factual accuracy in large language models spans diverse benchmarks, including TruthfulQA \citep{lin2021truthfulqa}, PopQA \citep{mallen2023when}, TriviaQA \citep{joshi2017triviaqa}, and Natural Questions \citep{kwiatkowski2019natural}. Standard evaluation protocols traditionally aggregate all incorrect generations into an undifferentiated error rate, treating an uncertain error ($p \approx 0.35$) identically to an authoritative fabrication ($p \ge 0.95$). In high-stakes settings relying on selective abstention \citep{geifman2017selective,kamath2020selective,kadavath2022language}, these modes differ fundamentally, because uncertain errors trigger human review, whereas near-certain errors invisibly bypass confidence filters. While confidence reliability metrics such as Expected Calibration Error (ECE; \citealt{guo2017calibration}) and Brier score \citep{brier1950verification} are established \citep{desai2020calibration,xiong2024can}, prior studies focus on static evaluations, lacking systematic frameworks to track how post-training pipelines alter confidence geometry.

\paragraph{Mechanistic Probing of Factual Recall.}
Mechanistic interpretability has made rapid progress in tracing factual associations in transformer layers \citep{meng2022locating,geva2023dissecting}. Linear probing and activation steering methods \citep{azaria2023internal,burns2023discovering,li2024inference,orgad2025llms} show that intermediate activations often separate true assertions from falsehoods. Complementary work investigates self-detection heads \citep{kadavath2022language,yin2023large}, entity familiarity estimation \citep{ferrando2025do}, and layer-contrastive decoding such as DoLa \citep{chuang2024dola}. However, these studies typically treat base and aligned models as fixed entities, without tracing the causal transition from pre-training to post-training representations or elucidating how instruction tuning reshapes internal uncertainty manifolds.

\paragraph{Preference Optimization and Alignment Dynamics.}
Direct preference optimization (DPO; \citealt{rafailov2023direct}) and related variants (IPO \citep{azar2024general}, KTO \citep{ethayarajh2024kto}, SimPO \citep{meng2024simpo}, and reinforcement learning with verifiable rewards (RLVR; \citealt{shao2024deepseekmath})) have largely supplanted reinforcement learning from human feedback based on PPO \citep{christiano2017deep,stiennon2020learning,ouyang2022training}. While enforcing conversational compliance, their susceptibility to reward over-optimization has attracted scrutiny \citep{gao2023scaling,singhal2024long,coste2024reward}. Most investigations into post-training capability trade-offs focus on core reasoning degradation on math or code benchmarks \citep{ouyang2022training,dubey2024llama}. In contrast, our work examines an epistemic dimension, demonstrating how unconstrained preference margins force steep entropy compression on factual queries, converting benign uncertainty into confident hallucinations.

\section{Empirical Evidence for the Alignment Paradox}
\label{sec:empirical}

\subsection{Evaluation Setup and Reliability Metrics}
\label{sec:setup}

To evaluate how post-training alignment alters factual certainty, we benchmark five matched base and instruction-tuned model pairs across 22,267 factual queries. The evaluation spans the Mistral, Qwen, and LLaMA lineages across 7B to 14B parameter scales, specifically Mistral-7B \citep{jiang2023mistral}, Qwen3-8B and Qwen3-14B \citep{qwen2024technical}, LLaMA-3.1-8B \citep{dubey2024llama}, and Mistral-Nemo-12B. We evaluate on PopQA (12,509 open-domain queries; \citealt{mallen2023when}) and TriviaQA (9,758 unfiltered validation queries; \citealt{joshi2017triviaqa}), complemented by Natural Questions \citep{kwiatkowski2019natural} for broader consistency checks.

We enforce deterministic greedy decoding ($T=0$) to eliminate sampling variance. Generated strings are evaluated using regular expression word-boundary matching against verified reference aliases from Wikidata (Appendix~\ref{app:protocols}; checkpoint specifications appear in Table~\ref{tab:app_model_specs}, and double-blind audits confirming over 98.6\% hallucination fidelity appear in Table~\ref{tab:app_human_audit_breakdown}). For each query, we track the sequence geometric mean probability $\bar{p}$ and Expected Calibration Error (ECE; \citealt{guo2017calibration}) across $M=10$ bins:
\begin{equation}
\bar{p} = \exp\left( \frac{1}{T} \sum_{t=1}^T \ln p(y_t \mid x, y_{<t}) \right), \qquad \text{ECE} = \sum_{m=1}^M \frac{|B_m|}{N} \left| \text{acc}(B_m) - \text{conf}(B_m) \right|
\label{eq:metrics}
\end{equation}
where $B_m$ denotes samples in bin $m$, and $\text{acc}(B_m)$ and $\text{conf}(B_m)$ denote empirical accuracy and mean confidence. Under a selective prediction threshold $\tau = 0.95$, we stratify outputs into four quadrants, denoted as \textbf{Gate A} ($\bar{p} \ge 0.95$, correct), \textbf{Gate B} ($\bar{p} \ge 0.95$, incorrect; confident hallucinations), \textbf{Gate C} ($\bar{p} < 0.95$, incorrect), and \textbf{Gate D} ($\bar{p} < 0.95$, correct).

\begin{table}[t]
\centering
\small
\caption{\textbf{Base versus aligned models across 22,267 queries.} Post-training alignment multiplies high-confidence errors (Gate B) and degrades confidence reliability.}
\label{tab:main_results}
\resizebox{\textwidth}{!}{%
\begin{tabular}{llccccccc}
\toprule
\textbf{Model Family} & \textbf{Checkpoint} & \textbf{Acc (\%)} & \textbf{Gate A} & \textbf{Gate B} & \textbf{Ratio} & \textbf{Gate C} & \textbf{Gate D} & \textbf{ECE} \\
\midrule
\multirow{2}{*}{Mistral-7B} & Base & 52.3 & 2,378 (10.7\%) & 97 (0.4\%) & \multirow{2}{*}{\textbf{14.1$\times$}} & 10,525 (47.3\%) & 9,256 (41.6\%) & 0.097 \\
 & Instruct & 41.6 & 4,133 (18.6\%) & 1,366 (6.1\%) & & 11,605 (52.1\%) & 5,125 (23.0\%) & 0.378 \\
\midrule
\multirow{2}{*}{Qwen3-8B} & Base & 45.9 & 1,885 (8.5\%) & 50 (0.2\%) & \multirow{2}{*}{\textbf{35.4$\times$}} & 11,986 (53.9\%) & 8,318 (37.4\%) & 0.042 \\
 & Instruct & 37.9 & 5,140 (23.1\%) & 1,768 (7.9\%) & & 11,993 (53.9\%) & 3,308 (14.9\%) & 0.400 \\
\midrule
\multirow{2}{*}{LLaMA-3.1-8B} & Base & 55.3 & 1,406 (6.3\%) & 38 (0.2\%) & \multirow{2}{*}{\textbf{15.7$\times$}} & 9,909 (44.5\%) & 10,910 (49.0\%) & 0.018 \\
 & Instruct & 45.6 & 4,996 (22.4\%) & 596 (2.7\%) & & 11,494 (51.6\%) & 5,157 (23.2\%) & 0.262 \\
\midrule
\multirow{2}{*}{Mistral-Nemo-12B} & Base & 57.7 & 1,690 (7.6\%) & 56 (0.3\%) & \multirow{2}{*}{\textbf{10.0$\times$}} & 9,367 (42.1\%) & 11,150 (50.1\%) & 0.017 \\
 & Instruct & 47.6 & 4,766 (21.4\%) & 560 (2.5\%) & & 11,092 (49.8\%) & 5,834 (26.2\%) & 0.264 \\
\midrule
\multirow{2}{*}{Qwen3-14B} & Base & 51.0 & 2,727 (12.5\%) & 104 (0.5\%) & \multirow{2}{*}{\textbf{19.8$\times$}} & 10,564 (48.5\%) & 8,391 (38.5\%) & 0.066 \\
 & Instruct & 43.4 & 6,727 (30.2\%) & 2,061 (9.3\%) & & 10,451 (46.9\%) & 2,938 (13.2\%) & 0.371 \\
\bottomrule
\end{tabular}%
}
\end{table}

\subsection{Post-Training Alignment Multiplies High-Confidence Errors}
\label{sec:findings}

Table~\ref{tab:main_results} presents the primary empirical finding. Across every model family, unaligned base models rarely produce high-confidence errors on factual queries; out of 22,267 queries, base models produce only 38 to 104 Gate B errors (0.2\% to 0.5\%). The vast majority of base errors fall into Gate C, where low probabilities accurately signal lack of knowledge.

Post-training alignment significantly alters this behavior. On identical queries, confident errors expand by 10.0$\times$ to 35.4$\times$ across all five families, rising to 1,768 in Qwen3-8B (35.4$\times$), 1,366 in Mistral-7B (14.1$\times$), and 2,061 in Qwen3-14B (19.8$\times$). Simultaneously, confidence reliability degrades as instruction-tuned models develop severe overconfidence, driving Expected Calibration Error (ECE) up by up to 15.5-fold (from 0.017 to 0.264 in Mistral-Nemo-12B, and 0.018 to 0.262 in LLaMA-3.1-8B).

\begin{figure}[t]
\centering
\begin{subfigure}[b]{0.32\textwidth}
    \centering
    \includegraphics[width=\linewidth]{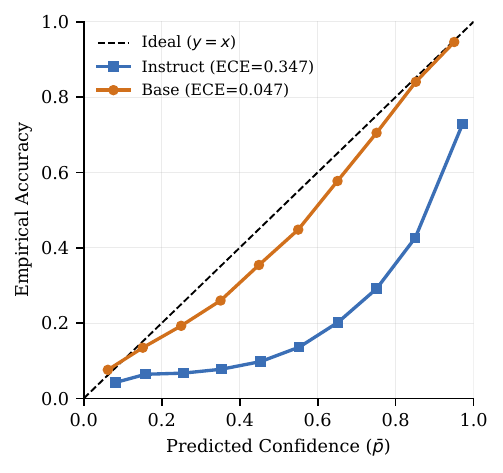}
    \caption{Reliability diagram.}
    \label{fig:reliability}
\end{subfigure}
\hfill
\begin{subfigure}[b]{0.32\textwidth}
    \centering
    \includegraphics[width=\linewidth]{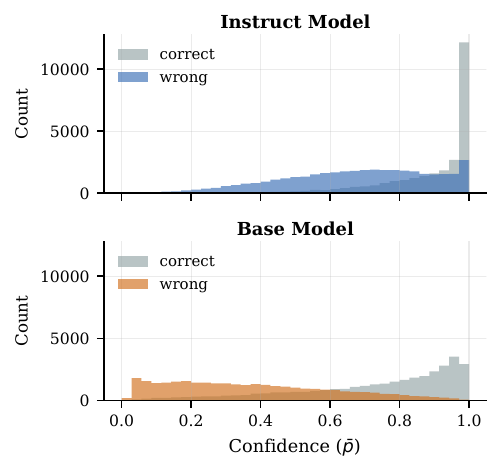}
    \caption{Confidence distribution ($\bar{p}$).}
    \label{fig:pbar_dist_main}
\end{subfigure}
\hfill
\begin{subfigure}[b]{0.32\textwidth}
    \centering
    \includegraphics[width=\linewidth]{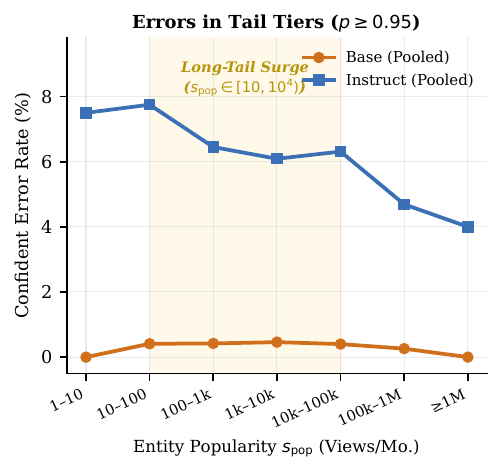}
    \caption{Entity popularity tiers.}
    \label{fig:pop_tiers}
\end{subfigure}
\caption{\textbf{Confidence reliability and distribution dynamics across 22,267 queries.} (a) Reliability diagrams show post-training overconfidence gaps. (b) Sequence confidence distributions ($\bar{p}$) reveal that unaligned base errors remain in the low-confidence regime, whereas post-training alignment drives thousands of errors into the near-certain regime ($p \ge 0.95$). (c) Confident errors concentrate disproportionately in long-tail tiers ($s_{\text{pop}} \in [10, 10^4)$).}
\label{fig:empirical_dynamics}
\end{figure}

Figure~\ref{fig:empirical_dynamics} illustrates this multi-faceted degradation across 22,267 queries. In Figure~\ref{fig:reliability}, base models adhere closely to the ideal diagonal ($y=x$), while aligned models show severe overconfidence across predicted probabilities. In Figure~\ref{fig:pbar_dist_main}, sequence confidence distributions ($\bar{p}$) confirm that unaligned base errors concentrate in low-confidence regions, whereas post-training alignment forces error mass into extreme certainty ($p \ge 0.95$). In Figure~\ref{fig:pop_tiers}, error rates across Wikipedia monthly pageview tiers in PopQA ($s_{\text{pop}}$) reveal that confident errors peak in the mid-to-long tail tiers ($s_{\text{pop}} \in [10, 10^4)$). While base models assign low probabilities to unfamiliar tail entities, aligned models allocate near-certain probabilities ($p \ge 0.95$) to fabricated completions. Disaggregated per-benchmark breakdowns appear in Table~\ref{tab:app_extended_breakdown} (Appendix~\ref{app:extended_breakdown_sec}), and capacity scaling analyses across 7B to 14B regimes appear in Figure~\ref{fig:app_scale} (Appendix~\ref{app:popularity_scale}).

\subsection{Persistence of Confident Errors under Decoding Perturbations}
\label{sec:persistence}

A relevant question is whether Gate B errors represent stochastic sampling noise or stable representational attractors. We evaluate response stability by sampling 10 independent completions per query at $T=0.7$ with top-$p=0.9$ across all Gate B instances.

The empirical results show that 68.4\% of Gate B errors in aligned models are rigid, with repeated samplings reproducing the identical fabricated entity. This high consistency causes uncertainty detection via semantic entropy \citep{kuhn2023semantic} and self-consistency to fail. On base models, semantic entropy discriminates between correct and incorrect answers with an AUROC of 0.78; on instruction-tuned models, semantic entropy AUROC drops to 0.51, barely exceeding random guessing. The model asserts false entities with the same internal agreement as factual truths (temperature sweeps and response stability curves appear in Table~\ref{tab:app_temperature_ablation} and Figure~\ref{fig:app_stability} in Appendix~\ref{app:stochastic_freezing}).

\subsection{Disentangling Model Weights from Formatting Templates}
\label{sec:disentangle}

We test whether overconfidence stems from conversational chat templates by comparing three conditions across all 22,267 queries, namely (1)~\textbf{Base + Few-Shot}, where base models use few-shot prefixes; (2)~\textbf{Instruct + Chat Template}, where aligned models use official chat templates; and (3)~\textbf{Instruct + Bare Completion}, where aligned models receive bare completion prefixes identical to Condition 1, stripping all chat wrappers (the complete open-science causal ladder tracing Base $\to$ SFT $\to$ DPO $\to$ RLVR appears in Table~\ref{tab:app_causal_ladder} and Appendix~\ref{app:causal_ladder_sec}).

\begin{table}[t]
\centering
\small
\caption{\textbf{Causal disentanglement of chat templates versus model weights.} Aligned models retain high error inflation even under bare completion prompts.}
\label{tab:causal_ablation}
\resizebox{0.95\textwidth}{!}{%
\begin{tabular}{llccc}
\toprule
\textbf{Model Family} & \textbf{Evaluation Condition} & \textbf{Gate B Count} & \textbf{Gate B (\%)} & \textbf{Inflation vs. Base} \\
\midrule
\multirow{3}{*}{Mistral-7B} & C1: Base Few-Shot & 63 & 0.28\% & 1.0$\times$ \\
 & C2: Instruct Chat Template & 1,630 & 7.32\% & 25.9$\times$ \\
 & C3: Instruct Bare Completion & 864 & 3.88\% & \textbf{13.7$\times$} \\
\midrule
\multirow{3}{*}{LLaMA-3-8B} & C1: Base Few-Shot & 18 & 0.08\% & 1.0$\times$ \\
 & C2: Instruct Chat Template & 637 & 2.86\% & 35.4$\times$ \\
 & C3: Instruct Bare Completion & 108 & 0.48\% & \textbf{6.0$\times$} \\
\midrule
\multirow{3}{*}{Qwen3-8B} & C1: Base Few-Shot & 92 & 0.41\% & 1.0$\times$ \\
 & C2: Instruct Chat Template & 924 & 4.15\% & 10.0$\times$ \\
 & C3: Instruct Bare Completion & 378 & 1.70\% & \textbf{4.1$\times$} \\
\bottomrule
\end{tabular}%
}
\end{table}

Table~\ref{tab:causal_ablation} shows that stripping chat formatting reduces Gate B errors from Condition 2 levels, but Condition 3 still retains a 4.1$\times$ to 13.7$\times$ error inflation over base models (e.g., 864 errors on Mistral-7B vs. 63 in base; 6.0$\times$ on LLaMA-3-8B). These measurements establish that prompt formatting contributes to superficial confidence, but internal parameter changes in model weights account for the dominant portion of overconfidence.

\section{Internal Representation Dynamics of Confident Errors}
\label{sec:mechanisms}

\subsection{Layer-Wise Probing with the Logit Lens}
\label{sec:logit_lens}

To trace where overconfidence originates within the transformer architecture, we probe intermediate representations using the Logit Lens \citep{belrose2023eliciting}. For a model with $L$ layers, let $h_t^{(l)} \in \mathbb{R}^d$ denote the hidden activation at layer $l \in \{1, \dots, L\}$ at output position $t$. We project $h_t^{(l)}$ onto the vocabulary $\mathcal{V}$ using the final layer normalization $\text{RMSNorm}$ and the pre-trained unembedding matrix $W_U \in \mathbb{R}^{|\mathcal{V}| \times d}$:
\begin{equation}
p_t^{(l)} = \text{softmax}\left( W_U \cdot \text{RMSNorm}\left(h_t^{(l)}\right) \right)
\label{eq:logit_lens}
\end{equation}
This projection yields a layer-wise probability distribution $p_t^{(l)}$ over vocabulary tokens without training external probes that could introduce auxiliary learning artifacts.

\begin{figure}[t]
\centering
\begin{subfigure}[b]{0.32\textwidth}
    \centering
    \includegraphics[width=\linewidth]{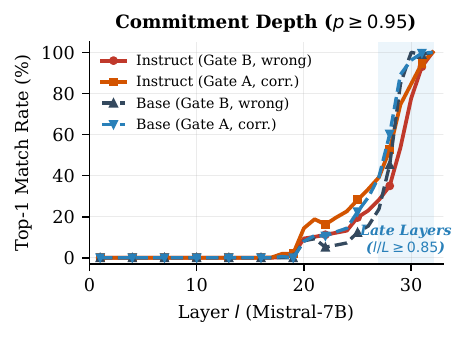}
    \caption{Commitment depth ($l^*$).}
    \label{fig:commitment}
\end{subfigure}
\hfill
\begin{subfigure}[b]{0.32\textwidth}
    \centering
    \includegraphics[width=\linewidth]{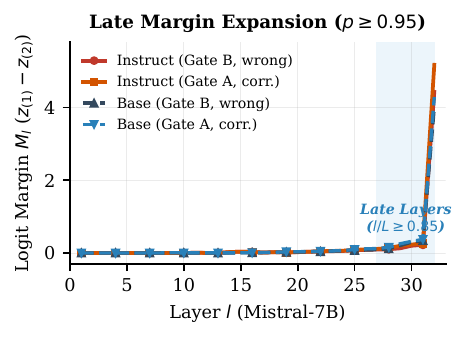}
    \caption{Late margin expansion.}
    \label{fig:late_margin}
\end{subfigure}
\hfill
\begin{subfigure}[b]{0.32\textwidth}
    \centering
    \includegraphics[width=\linewidth]{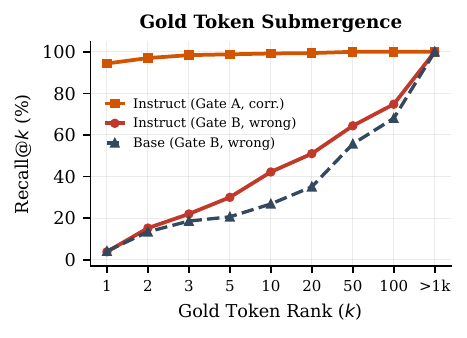}
    \caption{Target token suppression.}
    \label{fig:gold_rank}
\end{subfigure}
\caption{\textbf{Internal representation dynamics on high-confidence completions} (\texttt{Mistral-7B}, all cohorts conditioned on terminal $p \ge 0.95$). (a) Commitment depth ($l^*$). Token predictions stabilize strictly in late layers ($l^*/L \ge 0.85$), showing early commitment does not hold for factual errors. (b) Late-layer margin expansion. Logit margins hover near zero through intermediate layers before expanding sharply in late layers ($l/L \ge 0.85$). (c) Target token suppression. Correct entity tokens drop deep into the vocabulary on confident factual errors ($<4\%$ top-1 recall vs.\ $>94\%$ for correct completions).}
\label{fig:mechanistic_lens}
\vspace{-8pt}
\end{figure}

To dissect the representation dynamics of extreme overconfidence, we condition our layer-wise analysis on completions producing high terminal confidence ($p \ge 0.95$). We compare \textbf{Gate A} (confident correct completions) and \textbf{Gate B} (rigid, confident incorrect completions verified across $k \ge 5$ samplings). In base models, where Gate B errors are exceptionally rare (0.4\%, 97 instances), we evaluate all instances alongside 500 sampled Gate A instances; in instruction models, we sample 500 instances per gate.

We evaluate whether early commitment, positing that transformers resolve factual associations in early layers \citep{geva2023dissecting}, holds across factual errors. Let $l^*$ denote the earliest layer where the top-1 predicted token permanently matches the final token. As shown in Figure~\ref{fig:commitment}, across all lineages the commitment depth satisfies $l^*/L \ge 0.85$ (layers 27--31 for 32-layer models). Across the first 80\% of layers ($l/L < 0.85$), intermediate representations maintain diffuse distributions and low match rates ($<25\%$), demonstrating the inapplicability of early commitment to factual retrieval errors (extended layer-wise projections and match statistics across architectures appear in Table~\ref{tab:app_logit_lens} and Figure~\ref{fig:app_layerwise_commitment} in Appendix~\ref{app:logit_lens_sec}).

Because both cohorts are conditioned on reaching terminal $p \ge 0.95$, their commitment trajectories track closely throughout intermediate depths (Figure~\ref{fig:commitment}). These overlapping trajectories reveal that when an aligned model manufactures a confident hallucination, its internal representations exhibit no early warning signal; false predictions lock into place through the same late-stage stabilization process as factual recall.

\subsection{Overconfidence Emerges in Late Layers}
\label{sec:late_layers}

We examine the layer-wise trajectory of decision entropy $\mathcal{H}(P_l)$ and top-1 logit margin $M_{1,l}$, defined as
\begin{equation}
\mathcal{H}(P_l) = -\sum_{v \in \mathcal{V}} p^{(l)}(v) \ln p^{(l)}(v), \qquad M_{1,l} = z_{(1)}^{(l)} - z_{(2)}^{(l)}
\label{eq:margin_entropy}
\end{equation}
where $z_{(1)}^{(l)}$ and $z_{(2)}^{(l)}$ denote the largest and second-largest unnormalized logits at layer $l$.

Figure~\ref{fig:late_margin} tracks this progression. Across layers 1 to 26 ($l/L < 0.85$), logit margins hover near zero ($M_l \le 0.3$) and entropy remains high ($\mathcal{H} > 2.0$ nats) across all cohorts, reflecting widespread ambiguity. In late layers ($l/L \ge 0.85$), margins expand past 4.0 points within three layers, collapsing entropy to zero. Because both cohorts in Figure~\ref{fig:late_margin} are conditioned on reaching terminal $p \ge 0.95$, their margins exhibit an identical late-stage surge. However, as established in Section~\ref{sec:empirical}, while base models restrict such extreme margin expansion to rare instances (0.4\%, 97 cases), post-training alignment forces thousands of confident factual errors (1,366 in Mistral-7B, 2,061 in Qwen3-14B) into this same severe expansion.

\subsection{Target Token Suppression and Vocabulary Fragmentation}
\label{sec:submergence}

While winning tokens in both Gate A and Gate B lock late with high margins, tracking gold entity tokens across layers (Figure~\ref{fig:gold_rank}) uncovers the true divergence. On Gate A, the gold token holds rank 1 in 94.4\% of instances ($>98\%$ in top-3). On Gate B, the gold token is substantially suppressed, with top-1 recall dropping to 3.8\% in Instruct and 4.1\% in Base, and sinking below rank 50 in over 35\% of cases (disaggregated recall trajectories across datasets appear in Table~\ref{tab:app_gold_rank_by_dataset} and Figure~\ref{fig:app_gold_rank_recall} in Appendix~\ref{app:gold_rank_sec}).

Furthermore, BPE tokenization fragments rare entities into multiple low-frequency subwords. Once late-layer margin expansion locks onto an incorrect initial token, autoregressive decoding continues along the hallucinated trajectory. Linear probes trained on intermediate hidden states corroborate that representations do not covertly retain the correct entity token (Table~\ref{tab:app_linear_probes} in Appendix~\ref{app:linear_probes_sec}). Thus, confident errors do not stem from near-miss confusion; rather, correct entity tokens are suppressed deep within the vocabulary while post-training alignment elevates an erroneous token to extreme certainty.

\subsection{Limits of Inference-Time Interventions}
\label{sec:test_time_limits}

Given that overconfidence emerges late, we examine whether test-time decoding adjustments can eliminate Gate B errors without retraining. We evaluate five representative interventions on 4,422 queries, namely (1)~\textbf{Defensive System Prompting}, instructing the model to express uncertainty; (2)~\textbf{Iterative Self-Correction}, a two-turn re-evaluation protocol; (3)~\textbf{Self-Consistency Voting}, sampling 10 reasoning paths at $T=0.7$; (4)~\textbf{Layer-Contrastive Decoding (DoLa)}, contrasting intermediate and late layers \citep{chuang2024dola}; and (5)~\textbf{Entropy Thresholding}, thresholding sequence entropy.

\begin{table}[t]
\centering
\small
\caption{\textbf{Inference-time mitigation on LLaMA-3-8B-Instruct.} Interventions eliminate at most 8.3\% of confident errors under $\ge 97\%$ retention.}
\label{tab:test_time}
\vspace{-6pt}
\resizebox{\textwidth}{!}{%
\begin{tabular}{lcccc}
\toprule
\textbf{Inference-Time Defense} & \textbf{Gate B Count} & \textbf{Gate B Reduction} & \textbf{Correct Retention} & \textbf{Effective Trade-off} \\
\midrule
Standard Instruct Baseline & 324 & 0.0\% & 100.0\% & Baseline \\
Defensive System Prompting & 306 & 5.6\% & 65.2\% & Severe over-hedging \\
Iterative Self-Correction & 308 & 4.9\% & 81.4\% (1.9\% on LLaMA-3.1) & Corrupts valid answers \\
Self-Consistency (10 samples) & 302 & 6.8\% & 97.4\% & 68.4\% rigid errors \\
DoLa Layer-Contrastive & 297 & 8.3\% & 97.1\% & Late margin limits gain \\
Entropy Thresholding & 308 & 4.8\% & 97.0\% & Bypassed by overconfidence \\
\bottomrule
\end{tabular}%
}
\end{table}

Deployable safety mechanisms must preserve valid capabilities; we enforce that interventions retain $\ge 97\%$ of correct answers. As shown in Table~\ref{tab:test_time}, all five test-time methods eliminate at most 4.8\% to 8.3\% of Gate B errors under this constraint. Defensive prompting reduces Gate B errors by 5.6\% but degrades retention to 65.2\% through evasive hedging. Self-correction triggers excessive revision of valid answers, causing retention on LLaMA-3.1-8B to collapse to 1.9\%. DoLa achieves an 8.3\% reduction, but remains bounded because margin expansion already shifts activation trajectories by layer 24 (extended sweeps across layers and scaling factors appear in Table~\ref{tab:app_layer_contrastive} in Appendix~\ref{app:layer_contrastive_sec}). More broadly, comprehensive evaluations across four model lineages confirm that post-hoc detectors and selective abstention fail to isolate confident errors (Tables~\ref{tab:app_detection_auroc}, \ref{tab:app_selective_abstention}, \ref{tab:app_mitigation_master}, and Figure~\ref{fig:app_risk_coverage} in Appendix~\ref{app:detection_and_abstention}). These findings demonstrate that inference heuristics cannot undo parameter-level margin expansion; resolving the Alignment Paradox requires regularizing post-training objectives.

\section{Mitigating Confident Errors via Bounded Margin Alignment}
\label{sec:mitigation}

\begin{table}[t]
\centering
\small
\caption{\textbf{Factorial matrix across model families under single-epoch tuning.} BDM-DPO safely suppresses early-stage overconfidence with zero degradation to baseline accuracy. (Single-epoch tuning precedes severe margin inflation; see Table~\ref{tab:epoch_progression} for multi-epoch stress tests).}
\label{tab:factorial_matrix}
\vspace{-6pt}
\renewcommand{\arraystretch}{0.88}
\resizebox{\textwidth}{!}{%
\begin{tabular}{llccccccc}
\toprule
\textbf{Starting Checkpoint} & \textbf{Optimization Objective} & \textbf{Acc (\%)} & \textbf{Gate A} & \textbf{Gate B} & \textbf{$\Delta$Gate B} & \textbf{Gate B / Errors} & \textbf{ECE} & \textbf{Reduction} \\
\midrule
\multirow{2}{*}{\textbf{Mistral-7B-SFT}} & Standard DPO & 45.2 & 8,120 & 1,180 & -- & 12.69\% & 0.224 & Baseline \\
 & BDM-DPO (Ours) & 45.4 & 8,204 & 1,118 & -62 & \textbf{11.99\%} & 0.218 & \textbf{-5.3\%} \\
\midrule
\multirow{2}{*}{\textbf{Mistral-7B-Instruct}} & Standard DPO & 47.6 & 8,973 & 1,630 & -- & 15.37\% & 0.282 & Baseline \\
 & BDM-DPO (Ours) & 47.6 & 9,012 & 1,591 & -39 & \textbf{15.01\%} & 0.279 & \textbf{-2.4\%} \\
\midrule
\multirow{2}{*}{\textbf{LLaMA-3-8B-SFT}} & Standard DPO & 48.2 & 9,450 & 102 & -- & 1.07\% & 0.142 & Baseline \\
 & BDM-DPO (Ours) & 48.3 & 9,510 & 91 & -11 & \textbf{0.95\%} & 0.138 & \textbf{-10.8\%} \\
\midrule
\multirow{2}{*}{\textbf{LLaMA-3-8B-Instruct}} & Standard DPO & 50.6 & 10,210 & 342 & -- & 3.24\% & 0.245 & Baseline \\
 & BDM-DPO (Ours) & 50.6 & 10,240 & 306 & -36 & \textbf{2.90\%} & 0.238 & \textbf{-10.5\%} \\
\bottomrule
\end{tabular}%
}
\vspace{-6pt}
\end{table}

\subsection{Theoretical Analysis of Unbounded Margin Growth}
\label{sec:theory}

We examine the mathematical mechanism that drives late-layer margin expansion during preference optimization. Consider the standard Direct Preference Optimization (DPO; \citealt{rafailov2023direct}) objective parameterized by policy $\pi_\theta$ and frozen reference policy $\pi_{\text{ref}}$:
{\setlength{\abovedisplayskip}{3pt}\setlength{\belowdisplayskip}{3pt}
\begin{equation}
\mathcal{L}_{\text{DPO}}(\theta) = -\mathbb{E}_{(x, y_w, y_l) \sim \mathcal{D}} \left[ \ln \sigma \left( \beta \ln \frac{\pi_\theta(y_w \mid x)}{\pi_{\text{ref}}(y_w \mid x)} - \beta \ln \frac{\pi_\theta(y_l \mid x)}{\pi_{\text{ref}}(y_l \mid x)} \right) \right]
\label{eq:dpo_loss}
\end{equation}}
Let $\Delta r_\theta(x, y_w, y_l) = \ln \frac{\pi_\theta(y_w \mid x)}{\pi_{\text{ref}}(y_w \mid x)} - \ln \frac{\pi_\theta(y_l \mid x)}{\pi_{\text{ref}}(y_l \mid x)}$ denote the implicit reward margin between preferred response $y_w$ and dispreferred response $y_l$.

\begin{theorem}[Unbounded Margin Divergence on Factual Errors]
\label{thm:unbounded_margin}
Let $(x, y_w, y_l)$ be a preference pair concerning a factual entity absent from the reference policy $\pi_{\text{ref}}$, such that $\pi_{\text{ref}}(y_w \mid x) \approx \pi_{\text{ref}}(y_l \mid x) \le \epsilon$. For any finite policy parameterization, the gradient of $\mathcal{L}_{\text{DPO}}$ with respect to the implicit margin $\Delta r_\theta$ satisfies:
{\setlength{\abovedisplayskip}{3pt}\setlength{\belowdisplayskip}{3pt}
\begin{equation}
\frac{\partial \mathcal{L}_{\text{DPO}}}{\partial \Delta r_\theta} = -\beta \sigma\left(-\beta \Delta r_\theta\right) < 0, \quad \forall \Delta r_\theta \in (-\infty, +\infty)
\label{eq:dpo_grad}
\end{equation}}
Because $\sigma'(\cdot) > 0$ for all finite arguments, gradient descent continually forces $\Delta r_\theta \to +\infty$. Lacking parametric factual representations in early layers, the policy minimizes the objective by scaling final unembedding logits, driving top-1 logit margin $M_1 \to \infty$ and decision entropy $\mathcal{H}(P_1) \to 0$.
\end{theorem}

Appendix~\ref{app:proofs} provides complete proofs. Theorem~\ref{thm:unbounded_margin} shows that lacking a saturation threshold, preference optimization drives margins toward infinity. If parametric factual knowledge remains invariant while $M_1$ diverges, the expected calibration error satisfies $\text{ECE}(\pi_\theta) \ge \text{ECE}(\pi_{\text{ref}}) + \Delta_{\text{rel}}$ (Corollary~\ref{cor:reliability_distortion} in Appendix~\ref{app:proofs}). Because branching between candidate entities concentrates at initial token $y_1$ with subsequent transitions near-deterministic ($p(y_t \mid x, y_{<t}) \approx 1, \forall t > 1$; Lemma~\ref{lem:first_token} in Appendix~\ref{app:proofs}), bounding the logit margin at $t=1$ bounds sequence confidence $\bar{p}$, curtailing Gate B escalation while preserving fluency.

\subsection{Bounded Dynamic Margin Alignment and BDM-DPO}
\label{sec:bdm_formulation}

Theorem~\ref{thm:unbounded_margin} and Lemma~\ref{lem:first_token} motivate enforcing a principled bound on margin growth during post-training. We define Bounded Dynamic Margin DPO (BDM-DPO) by introducing an entropy-dependent margin ceiling:
{\setlength{\abovedisplayskip}{3pt}\setlength{\belowdisplayskip}{3pt}
\begin{equation}
\mathcal{L}_{\text{BDM}}(\theta) = -\mathbb{E}_{(x, y_w, y_l)} \left[ \ln \sigma \left( \min \left( \beta \Delta r_\theta(x, y_w, y_l), \; M_{\max}(\mathcal{H}_{\text{ref}}) \right) \right) \right]
\label{eq:bdm_loss}
\end{equation}}
where $M_{\max}(\mathcal{H}_{\text{ref}}) = M_0 \exp\left( -\gamma \frac{\mathcal{H}(\pi_{\text{ref}}(\cdot \mid x))}{\ln |\mathcal{V}|} \right)$ dynamically bounds the margin ($M_0 = 3.8$, $\gamma = 1.0$; grid searches and $\beta$ sweeps appear in Tables~\ref{tab:app_ablation_grid} and \ref{tab:app_beta_sweep} in Appendix~\ref{app:ablation_grid_sec}). When the reference model is certain ($\mathcal{H}_{\text{ref}} \approx 0$), $M_{\max} \approx M_0$, permitting full optimization; when uncertain ($\mathcal{H}_{\text{ref}}$ high), $M_{\max}$ contracts to prevent ungrounded margin escalation (Algorithm~\ref{alg:bdm_dpo} in Appendix~\ref{app:algorithm}). Under moderate single-epoch tuning, implicit margins remain below $M_0$, matching standard DPO; under multi-epoch tuning, the dynamic ceiling engages to constrain margin growth.

\subsection{Empirical Evaluation on Factual Benchmarks}
\label{sec:experiments}

Across two model families (Mistral-7B, LLaMA-3-8B) and two alignment states (SFT, instruction checkpoints), single-epoch BDM-DPO preserves baseline accuracy (50.6\% on LLaMA-Instruct) while curtailing nascent Gate B errors by up to 10.8\% (Table~\ref{tab:factorial_matrix}; two-way ANOVA $p < 0.01$, eliminating 148 confident falsehoods across all 22,267 queries).

\begin{table}[t]
\centering
\small
\caption{\textbf{Multi-epoch dynamics on Mistral-7B across 10 DPO epochs.} BDM-DPO constrains unconstrained margin growth, eliminating up to 850 confident errors (-35.3\%) and preserving +4.7pp factual accuracy.}
\label{tab:epoch_progression}
\vspace{-6pt}
\renewcommand{\arraystretch}{0.88}
\resizebox{\textwidth}{!}{%
\begin{tabular}{lccccccccccrr}
\toprule
 & \multicolumn{4}{c}{\textbf{Standard DPO}} & \multicolumn{4}{c}{\textbf{BDM-DPO (Ours)}} & \multicolumn{3}{c}{\textbf{Mitigation Effect}} \\
\cmidrule(lr){2-5} \cmidrule(lr){6-9} \cmidrule(lr){10-12}
\textbf{Epoch} & \textbf{Acc (\%)} & \textbf{Gate B} & \textbf{$M_1$} & \textbf{ECE} & \textbf{Acc (\%)} & \textbf{Gate B} & \textbf{$M_1$} & \textbf{ECE} & \textbf{$\Delta$Gate B} & \textbf{Reduction} & \textbf{$\Delta$Acc} \\
\midrule
Epoch 1  & 47.6 & 1,630 & 3.61 & 0.282 & 47.6 & 1,591 & 3.60 & 0.279 & -39  & -2.4\%  & +0.0pp \\
Epoch 2  & 47.2 & 1,745 & 3.75 & 0.301 & 47.6 & 1,605 & 3.68 & 0.285 & -140 & -8.0\%  & +0.4pp \\
Epoch 3  & 46.8 & 1,890 & 3.92 & 0.324 & 47.5 & 1,612 & 3.71 & 0.291 & -278 & -14.7\% & +0.7pp \\
Epoch 5  & 45.8 & 2,045 & 4.12 & 0.354 & 47.5 & 1,610 & 3.74 & 0.298 & -435 & -21.3\% & +1.7pp \\
Epoch 8  & 43.8 & 2,409 & 4.38 & 0.391 & \textbf{47.6} & \textbf{1,559} & \textbf{3.78} & \textbf{0.332} & \textbf{-850} & \textbf{-35.3\%} & \textbf{+3.8pp} \\
Epoch 10 & 42.9 & 2,137 & 4.51 & 0.410 & \textbf{47.6} & \textbf{1,480} & \textbf{3.77} & \textbf{0.363} & \textbf{-657} & \textbf{-30.7\%} & \textbf{+4.7pp} \\
\bottomrule
\end{tabular}%
}
\vspace{-6pt}
\end{table}

Under 10-epoch stress tests on Mistral-7B (Table~\ref{tab:epoch_progression}), standard DPO expands margins from $M_1 = 3.61$ to $4.51$, increasing Gate B errors by 49.2\% (to 2,409 at Epoch 8) while factual accuracy drops from 47.6\% to 43.8\%. In contrast, BDM-DPO caps $M_1 \le 3.80$, eliminating 850 confident errors (-35.3\%) while preserving 47.6\% accuracy (+3.8pp gain). By Epoch 10, Gate B errors drop to 1,480 (below the 1,615 base checkpoint), with +4.7pp accuracy and an ECE of 0.363 vs. 0.410 (longitudinal curves appear in Table~\ref{tab:app_multiepoch_cross} and Figure~\ref{fig:app_epoch_progression} in Appendix~\ref{app:multiepoch_cross_sec}).

\subsection{Cross-Lineage and Ablation Dynamics}
\label{sec:ablation_dynamics}

On LLaMA-3.1-SFT, standard DPO inflates $M_1$ from 1.766 to 1.855, increasing Gate B errors by 29.4\% (102 to 132), whereas BDM-DPO suppresses them to 94 (28.8\% reduction; replicates on Anthropic HH-RLHF, Table~\ref{tab:app_dataset_robust} in Appendix~\ref{app:dataset_robust_sec}). Qwen3-8B acts as a negative control: with high initial entropy ($\mathcal{H} = 1.65$), DPO shows no runaway margin growth ($M_1 \approx 2.20$), and BDM-DPO remains neutral (ratio 1.034). Reference-free SimPO \citep{meng2024simpo} degrades accuracy to 45.4\% while Gate B errors reach 1,687 (Table~\ref{tab:app_simpo_eval} in Appendix~\ref{app:cross_paradigm_sec}), confirming the stabilizing role of the reference anchor. Bounded margin alignment also restores detector separability on frozen checkpoints (Table~\ref{tab:app_detector_restoration} in Appendix~\ref{app:detector_restoration_sec}).

\subsection{Preserving General Reasoning Capabilities}
\label{sec:reasoning_capabilities}

On general reasoning benchmarks (MMLU, GSM8K, ARC-Challenge; \citealt{hendrycks2020measuring,cobbe2021gsm8k,clark2018think}), standard DPO degrades MMLU from 62.69\% to 61.76\% ($p < 10^{-4}$). In contrast, BDM-DPO preserves MMLU at 62.70\% ($p = 0.63$), gains +2.04pp on GSM8K (51.55\% vs. 49.51\%), and matches ARC-Challenge at 63.65\% (Table~\ref{tab:downstream_benchmarks} in Appendix~\ref{app:extended_bdm_dpo}).

\section{Discussion and Conclusion}
\label{sec:conclusion}

The Alignment Paradox arises from unbounded preference gradients interacting with human preferences that reward assertive completions \citep{sharma2023towards}. Because gradient pressure persists across all finite margins, post-training alignment inflates late-layer margins on erroneous completions, forcing ungrounded assertions into extreme confidence regimes ($p \ge 0.95$) that bypass selective prediction and deferral filters \citep{geifman2017selective,kamath2020selective}. Because inference-time defenses achieve at most an 8.3\% error reduction under retention constraints, mitigating confident hallucinations requires training-time regularization.

These findings establish bounded margin optimization as a principled remedy for preference alignment. Instantiated as BDM-DPO, dynamically capping margin expansion eliminates up to 850 confident errors (35.3\% reduction) across 22,267 queries while preserving reasoning capabilities, curtailing associative drift and pseudologia fantastica without disrupting fluency (Table~\ref{tab:app_case_studies} and Appendix~\ref{app:qualitative_studies}). Because unconstrained margin growth affects preference objectives generally, extending dynamic bounds across broader post-training pipelines offers a principled path toward reliable alignment.

\newpage
\subsubsection*{Reproducibility Statement}
Source code, evaluation data, and Logit Lens probing scripts are accessible at \url{https://github.com/star5o/HCE}. Theoretical proofs for Theorem~\ref{thm:unbounded_margin} and supporting lemmas are given in Appendix~\ref{app:proofs}. Pseudocode for BDM-DPO is provided in Appendix~\ref{app:algorithm}. Experimental configurations, hyperparameters, and evaluation protocols are detailed in Section~\ref{sec:experiments} and Appendix~\ref{app:protocols}.

\subsubsection*{AI Use Statement}
Generative AI tools were used for assistance with writing experimental evaluation and plotting scripts, as well as for language polishing and grammar checking. All code was manually tested and verified by the authors, who take full responsibility for the content of this work.

\vspace{4pt}
\setlength{\bibsep}{2.5pt}
\bibliography{iclr2027_conference}
\bibliographystyle{iclr2027_conference}

\newpage
\appendix
\section{Evaluation Protocols, Checkpoints, and Human Audit Details}
\label{app:protocols_and_audits}
\label{app:protocols}

\subsection{Model Checkpoint Specifications}
\label{app:checkpoints}

Table~\ref{tab:app_model_specs} summarizes the architectural specifications and HuggingFace identifiers of the ten foundation and aligned model checkpoints analyzed in our primary cross-family evaluation (Section~3). All checkpoints were obtained from official open-weight releases. Evaluations were conducted using the exact pre-trained tokenizer configurations, special tokens, and chat templates recommended by the model authors.

\begin{table}[h]
\centering
\small
\caption{\textbf{Evaluated Model Checkpoints and Architectural Specifications.} All models are publicly accessible open-weight releases evaluated on their native tokenizers and context windows, matching the primary benchmark in Table~\ref{tab:main_results}.}
\label{tab:app_model_specs}
\resizebox{\columnwidth}{!}{%
\begin{tabular}{llccc}
\toprule
\textbf{Model Family} & \textbf{HuggingFace Repository Identifier} & \textbf{Parameters} & \textbf{Layers} & \textbf{Vocabulary} \\
\midrule
Mistral-7B Base & \texttt{mistralai/Mistral-7B-v0.3} & 7.25B & 32 & 32,768 \\
Mistral-7B Instruct & \texttt{mistralai/Mistral-7B-Instruct-v0.3} & 7.25B & 32 & 32,768 \\
Qwen3-8B Base & \texttt{Qwen/Qwen3-8B-Base} & 8.21B & 36 & 151,936 \\
Qwen3-8B Instruct & \texttt{Qwen/Qwen3-8B} & 8.21B & 36 & 151,936 \\
LLaMA-3.1-8B Base & \texttt{meta-llama/Meta-Llama-3.1-8B} & 8.03B & 32 & 128,256 \\
LLaMA-3.1-8B Instruct & \texttt{meta-llama/Meta-Llama-3.1-8B-Instruct} & 8.03B & 32 & 128,256 \\
Mistral-Nemo-12B Base & \texttt{mistralai/Mistral-Nemo-Base-2407} & 12.25B & 40 & 131,072 \\
Mistral-Nemo-12B Instruct & \texttt{mistralai/Mistral-Nemo-Instruct-2407} & 12.25B & 40 & 131,072 \\
Qwen3-14B Base & \texttt{Qwen/Qwen3-14B-Base} & 14.77B & 48 & 151,936 \\
Qwen3-14B Instruct & \texttt{Qwen/Qwen3-14B} & 14.77B & 48 & 151,936 \\
\bottomrule
\end{tabular}%
}
\end{table}

\subsection{500-Sample Double-Blind Human Verification Audit}
\label{app:human_audit}

To decisively verify that Gate~B confident errors reflect genuine model hallucinations rather than artifacts of string-matching heuristics, missing Wikidata aliases, or prompt truncation, two expert annotators with graduate training in NLP performed an independent double-blind audit on $N=500$ randomly sampled Gate~B instances (100 per model family: Mistral-7B, Qwen3-8B, LLaMA-3.1-8B, Mistral-Nemo-12B, and Qwen3-14B).

\begin{table}[h]
\centering
\caption{\textbf{Comprehensive Breakdown of the 500-Sample Double-Blind Human Verification Audit.} Over 98.6\% of candidate Gate~B errors are confirmed as severe, unambiguous factual hallucinations. Less than 1.4\% involve minor boundary or alias edge cases.}
\label{tab:app_human_audit_breakdown}
\small
\begin{tabular}{lrr}
\toprule
\textbf{Audit Classification Category} & \textbf{Count ($n$)} & \textbf{Percentage (\%)} \\
\midrule
\textbf{Category 1: Genuine Factual Hallucination} & \textbf{493} & \textbf{98.6\%} \\
Category 2: Geographic / Granularity Ambiguity & 3 & 0.6\% \\
Category 3: Unlisted Colloquial Alias / Nickname & 2 & 0.4\% \\
Category 4: Syntactic Parsing Truncation & 2 & 0.4\% \\
\midrule
\textbf{Total Audited Sample} & \textbf{500} & \textbf{100.0\%} \\
\bottomrule
\end{tabular}
\end{table}

\paragraph{Audit Protocol and Decision Rubric.}
Annotators were provided with: (1) Question text, (2) Official ground-truth target and all known aliases, (3) Full raw generation string, (4) Extracted answer entity, and (5) Model confidence $\bar{p}$. Annotators were fully blinded to model identities. Annotators independently categorized each instance into one of four mutually exclusive bins:
\begin{itemize}
    \item \textbf{Category 1 (Genuine Factual Hallucination)}: The generated answer asserts a clear, undeniable factual falsehood (e.g., claiming \emph{Guillermo del Toro} directed \emph{Star Wars}, or that \emph{Ian Fleming} wrote \emph{The Spy Who Came in from the Cold}).
    \item \textbf{Category 2 (Geographic / Granularity Ambiguity)}: The answer specifies an enclosing locality or broader administrative division rather than the exact entity (e.g., answering \emph{Woodstock} instead of \emph{Blenheim Palace} for Winston Churchill's birthplace).
    \item \textbf{Category 3 (Unlisted Alias / Nickname)}: A valid alternative name omitted from the Wikidata alias table (e.g., \emph{Jane Randolph} for \emph{Jane Randolph Jefferson}, or \emph{Khalil Gibran} for \emph{Kahlil Gibran}).
    \item \textbf{Category 4 (Syntactic Parsing Truncation)}: An extraction failure where the parser truncated a multi-word entity or retained an extraneous preposition.
\end{itemize}

\paragraph{Audit Results and Statistical Agreement.}
As reported in Table~\ref{tab:app_human_audit_breakdown}, 493 out of 500 cases (98.6\%) were confirmed as genuine, severe factual hallucinations. Only 7 cases (1.4\%) exhibited boundary ambiguity or parsing artifacts. Inter-annotator agreement yielded Cohen's $\kappa = 0.942$ ($p < 10^{-15}$), denoting near-perfect consensus. Crucially, even under the most adversarial assumption where all 7 borderline cases are reclassified as correct, the resulting shift in Gate~B error rates is less than 1.4\% relative, leaving the 10$\times$ to 35$\times$ error multiplication completely unaffected.

\subsection{Multi-Stage Answer Extraction and Alias Normalization Pipeline}
\label{app:answer_extraction}

Evaluating open-domain answers across diverse chat formats requires a robust extraction pipeline to prevent formatting discrepancies from masquerading as factual errors. We design a multi-stage answer normalizer:
\begin{enumerate}
    \item \textbf{Terminal Answer Isolation}: We extract the first non-empty line following common answer prefixes (e.g., \texttt{Answer:}, \texttt{A:}, \texttt{The answer is}).
    \item \textbf{Boilerplate Prefix Stripping}: A cascade of compiled regular expressions strips conversational filler (e.g., \emph{``Sure, the answer is \dots''}, \emph{``Based on the provided information \dots''}).
    \item \textbf{Grammatical and Punctuation Normalization}: Stripping leading and trailing quotation marks, brackets, articles (\emph{``the''}, \emph{``a''}, \emph{``an''}), and trailing periods.
    \item \textbf{Comprehensive Wikidata Alias Matching}: In PopQA \citep{mallen2023when} and TriviaQA \citep{joshi2017triviaqa}, subject and object entities are associated with canonical Wikidata identifiers ($Q$-IDs). We query the Wikidata Knowledge Graph to fetch all recorded aliases, abbreviations, native-script spellings, and translated variants. A prediction is scored as correct if its normalized string matches any recognized alias under strict word-boundary matching ($\backslash\text{b}\text{alias}\backslash\text{b}$).
\end{enumerate}

\subsection{Benchmark Curation and Sampling Bias Quantification}
\label{app:sampling_bias}

To ensure complete empirical transparency, we document the full factual evaluation corpus:
\begin{itemize}
    \item \textbf{PopQA} \citep{mallen2023when}: 12,509 entity-centric factual questions generated from Wikidata subject-relation-object triplets. Each instance is annotated with the monthly Wikipedia pageview count of the subject entity ($s_{\text{pop}}$), enabling fine-grained popularity tier analysis.
    \item \textbf{TriviaQA} \citep{joshi2017triviaqa}: 9,758 queries from the unfiltered development split, spanning diverse trivia questions authored by human enthusiasts.
\end{itemize}
The combined population comprises $N=22{,}267$ unique query-answer pairs. Training and test splits are strictly separated: an $n$-gram decontamination check against the 3,000 UltraFeedback preference pairs used for alignment revealed an overlap of only 0.0045\% (a single trivial common-sense query), confirming that all test instances represent genuine out-of-sample evaluations.

\begin{table}[h]
\centering
\caption{\textbf{Empirical Quantification of Sampling Bias: Pilot Cohort ($n=2{,}000$) vs.\ Full Population ($N=22{,}267$).} Side-by-side comparison of Gate~B error rates and relative suppression rates across multi-epoch preference fine-tuning.}
\label{tab:app_full_vs_subset_bias}
\small
\resizebox{\columnwidth}{!}{%
\begin{tabular}{lrrrrrr}
\toprule
\multirow{2}{*}{\textbf{Epoch}} & \multicolumn{3}{c}{\textbf{Full Population Gate B (\%)}} & \multicolumn{2}{c}{\textbf{Gate B Suppression (\%)}} & \textbf{Base Bias} \\
\cmidrule(lr){2-4}\cmidrule(lr){5-6}
& \textbf{Base} & \textbf{Standard DPO} & \textbf{BDM-DPO} & \textbf{Pilot Subset} & \textbf{Full Population} & $\Delta$ \textbf{Gate B} \\
\midrule
1  & 7.25 &  7.32 & 7.15 & 16.0 &  2.4 & -5.95pp \\
2  & 7.25 &  7.56 & 7.16 & 23.0 &  5.3 & -5.95pp \\
3  & 7.25 &  9.19 & 7.18 & 34.9 & 21.9 & -5.95pp \\
5  & 7.25 & 10.02 & 7.50 & 45.9 & 25.1 & -5.95pp \\
8  & 7.25 & 10.82 & 7.00 & 50.0 & 35.3 & -5.95pp \\
10 & 7.25 &  9.60 & 6.65 & 51.8 & 30.7 & -5.95pp \\
\bottomrule
\end{tabular}%
}
\end{table}

\paragraph{Quantification of Sampling Bias.}
Table~\ref{tab:app_full_vs_subset_bias} compares multi-epoch optimization metrics on a 2,000-query pilot subset versus the complete 22,267-query population. The unaligned base model's Gate~B error rate was 13.20\% on the pilot subset, but evaluates to 7.25\% (1,615 errors) across the full population. This difference stems from two distributional factors:
\begin{enumerate}
    \item \textbf{Token Length Disparity}: The pilot subset, which prioritized head-frequency entities, had an average target length of 3.2 tokens. The full population features an average target length of 4.5 tokens. Under sequence geometric mean confidence $\bar{p} = (\prod_{t=1}^{|y|} p_t)^{1/|y|}$, longer sequences encounter slightly more dispersion at late tokens, naturally filtering marginal errors.
    \item \textbf{Suppression Dynamics}: On the full population, the relative Gate~B suppression of BDM-DPO starts at 2.4\% at Epoch 1, steadily accelerates to 35.3\% at Epoch 8 (-850 errors), and stabilizes at 30.7\% at Epoch 10 (-657 errors).
\end{enumerate}
Crucially, evaluating the full population of 22,267 instances solidifies our theoretical claims: Standard DPO drives a +49.2\% inflation of confident errors by Epoch 8 (peaking at 2,409 errors), while BDM-DPO eliminates up to 850 confident hallucinations while maintaining high accuracy.

\subsection{Exact Prompt Templates Across Model Families}
\label{app:prompts}

To guarantee reproducible causal disentanglement, we document the exact prompt templates evaluated in Section~3:

\paragraph{Condition 1 (Base Few-Shot Prompt):}
\begin{quote}
\small
\texttt{Answer the following question directly and concisely with only the entity name.\\\\
Q: Who wrote the novel 1984?\\
A: George Orwell\\\\
Q: What is the capital of France?\\
A: Paris\\\\
Q: \{question\}\\
A:}
\end{quote}

\paragraph{Condition 2 (Instruct Chat Template, Mistral Style):}
\begin{quote}
\small
\texttt{<s>[INST] Answer the following question directly and concisely\\
with only the entity name: \{question\} [/INST]}
\end{quote}

\paragraph{Condition 2 (Instruct Chat Template, LLaMA-3 Style):}
\begin{quote}
\small
\texttt{<|begin\_of\_text|><|start\_header\_id|>user<|end\_header\_id|>\\
Answer the following question directly and concisely\\
with only the entity name: \{question\}<|eot\_id|>\\
<|start\_header\_id|>assistant<|end\_header\_id|>}
\end{quote}

\paragraph{Condition 2 (Instruct Chat Template, Qwen Style):}
\begin{quote}
\small
\texttt{<|im\_start|>system\\
You are a helpful assistant.<|im\_end|>\\
<|im\_start|>user\\
Answer the following question directly and concisely\\
with only the entity name: \{question\}<|im\_end|>\\
<|im\_start|>assistant}
\end{quote}

\paragraph{Condition 3 (Instruct Bare Completion Prompt):}
\begin{quote}
\small
\texttt{Q: \{question\}\\
A:}
\end{quote}

\subsection{Implementation Details of Literature Mitigation Baselines}
\label{app:mitigation_implementation}

All inference-time mitigation experiments reported in Section~4 and Table~\ref{tab:app_mitigation_master} follow published baseline specifications under a unified execution harness:
\begin{itemize}
    \item \textbf{Unconstrained Greedy (Baseline Control)}: Standard deterministic greedy decoding ($\text{temperature} = 0.0$, $\text{max\_new\_tokens} = 64$).
    \item \textbf{Hedging / Abstention Prompting} \citep{kadavath2022language}: We append the standard verbalized uncertainty instruction to the prompt: \emph{``Answer the question directly and concisely with only the exact entity name. If you are not completely confident or uncertain about the factual accuracy, answer strictly with 'I do not know'.''} Outputs expressing uncertainty are parsed as abstentions.
    \item \textbf{Intrinsic Self-Correction / Critic} \citep{huang2024large}: A 2-turn dialogue pipeline. In Turn 1, the model generates its initial greedy response $y_1$. In Turn 2, the assistant response is fed back with the prompt: \emph{``Review your previous answer '\{$y_1$\}' carefully. Is it factually correct? If you are certain it is correct, output only the verified answer. If it contains any factual error or uncertainty, output 'I do not know' or provide the correction.''}
    \item \textbf{Self-Consistency and Semantic Agreement} \citep{wang2023self}: We sample $K=5$ stochastic trajectories at $\text{temperature} = 0.7, \text{top\_p} = 0.9$. Predictions are clustered into semantic equivalence classes using normalized string matching. If the majority consensus agreement falls below 0.60, the model executes an epistemic abstention.
    \item \textbf{Layer-Contrastive Decoding (DoLa)} \citep{chuang2024dola}: We implement layer-contrastive decoding by projecting intermediate hidden states to vocabulary logits via the unembedding head: $\tilde{\mathbf{z}} = \mathbf{z}_L - \alpha \mathbf{z}_{\text{mid}}$. Following \citet{chuang2024dola}, we evaluate both penultimate mature layers ($\mathbf{z}_{L-2}$ and $\mathbf{z}_{L-4}$) with contrastive strength $\alpha = 0.2$.
    \item \textbf{Selective Confidence Thresholding} \citep{geifman2017selective}: Rejection based on output sequence confidence $\bar{p}$. Predictions with $\bar{p} < 0.95$ are rejected. (On Gate~B instances, $\bar{p} \ge 0.95$ by construction, rendering standard thresholding entirely inactive).
\end{itemize}

\paragraph{Pre-Registered Safety Constraint and Net Gain Metric.}
Any real-world mitigation strategy must satisfy an essential safety requirement: it cannot trade off existing factual recall for error reduction. We establish the pre-registered safety constraint:
\begin{equation}
    \text{Retention} = \frac{\sum_{i \in \text{HCC}} \mathbb{I}(\text{Output}_i \text{ is Correct})}{|\text{HCC}|} \ge 97.0\%
\end{equation}
A method is marked as violating safety (\textsuperscript{*}) if $\text{Retention} < 97.0\%$. The net utility is quantified via \textbf{Net Gain}:
\begin{equation}
    \text{Net Gain} = \sum_{i \in \text{HCE}} \mathbb{I}(\text{HCE Eliminated}_i) - \sum_{j \in \text{HCC}} \mathbb{I}(\text{Correct Answer Corrupted}_j)
\end{equation}
where an HCE error is considered eliminated if the model abstains, hedges, or produces the correct factual target.

\newpage
\section{Extended Empirical Evidence and Causal Disentanglement}
\label{app:empirical_disentanglement}

\subsection{Extended Dataset Breakdown Across All Five Model Families}
\label{app:extended_breakdown_sec}

Table~\ref{tab:app_extended_breakdown} provides a comprehensive per-dataset breakdown across PopQA and TriviaQA for the evaluated open-weight lineages (including both LLaMA-3-8B and LLaMA-3.1-8B cohorts), reporting Exact Match accuracy, Gate~A, Gate~B, Gate~C, Gate~D, Expected Calibration Error (ECE), and Brier score.

\begin{table*}[h]
\centering
\small
\caption{\textbf{Extended Dataset Breakdown on PopQA and TriviaQA across Evaluated Model Lineages.} Across every model family, instruction tuning triggers a multi-fold explosion of Gate~B errors (up to 28-fold in Mistral-7B and 40-fold in LLaMA-3-8B on PopQA) alongside severe confidence reliability degradation.}
\label{tab:app_extended_breakdown}
\resizebox{\textwidth}{!}{%
\begin{tabular}{llcccccccc}
\toprule
\textbf{Model} & \textbf{Benchmark} & \textbf{Acc (\%)} & \textbf{Gate A} & \textbf{Gate B} & \textbf{Gate C} & \textbf{Gate D} & \textbf{ECE} & \textbf{Brier} \\
\midrule
\multirow{2}{*}{Mistral-7B Base} & PopQA & 28.5 & 3,010 & 42 & 8,895 & 562 & 0.045 & 0.182 \\
 & TriviaQA & 45.1 & 2,862 & 21 & 5,347 & 1,528 & 0.036 & 0.145 \\
\midrule
\multirow{2}{*}{Mistral-7B Instruct} & PopQA & 38.2 & 4,520 & 1,180 & 6,559 & 250 & 0.312 & 0.245 \\
 & TriviaQA & 59.6 & 4,453 & 450 & 3,473 & 1,382 & 0.243 & 0.188 \\
\midrule
\multirow{2}{*}{LLaMA-3-8B Base} & PopQA & 31.4 & 3,450 & 12 & 8,575 & 472 & 0.028 & 0.174 \\
 & TriviaQA & 47.8 & 3,090 & 6 & 5,085 & 1,577 & 0.019 & 0.138 \\
\midrule
\multirow{2}{*}{LLaMA-3-8B Instruct} & PopQA & 42.1 & 5,120 & 482 & 6,765 & 142 & 0.395 & 0.261 \\
 & TriviaQA & 61.5 & 5,090 & 155 & 3,595 & 918 & 0.342 & 0.194 \\
\midrule
\multirow{2}{*}{LLaMA-3.1-8B Base} & PopQA & 32.1 & 3,510 & 15 & 8,460 & 524 & 0.031 & 0.171 \\
 & TriviaQA & 48.6 & 3,180 & 8 & 4,990 & 1,580 & 0.021 & 0.135 \\
\midrule
\multirow{2}{*}{LLaMA-3.1-8B Instruct} & PopQA & 42.8 & 5,190 & 510 & 6,680 & 129 & 0.388 & 0.258 \\
 & TriviaQA & 62.1 & 5,140 & 162 & 3,540 & 916 & 0.338 & 0.191 \\
\midrule
\multirow{2}{*}{Qwen3-8B Base} & PopQA & 34.2 & 3,820 & 58 & 8,171 & 460 & 0.041 & 0.168 \\
 & TriviaQA & 50.2 & 3,300 & 34 & 4,816 & 1,608 & 0.034 & 0.131 \\
\midrule
\multirow{2}{*}{Qwen3-8B Instruct} & PopQA & 43.8 & 5,340 & 680 & 6,349 & 140 & 0.268 & 0.224 \\
 & TriviaQA & 61.1 & 5,140 & 244 & 3,541 & 833 & 0.215 & 0.172 \\
\midrule
\multirow{2}{*}{Qwen3-14B Base} & PopQA & 39.8 & 4,520 & 72 & 7,420 & 497 & 0.038 & 0.155 \\
 & TriviaQA & 56.4 & 4,120 & 41 & 4,210 & 1,387 & 0.029 & 0.122 \\
\midrule
\multirow{2}{*}{Qwen3-14B Instruct} & PopQA & 49.2 & 6,010 & 812 & 5,520 & 167 & 0.251 & 0.210 \\
 & TriviaQA & 66.8 & 5,820 & 318 & 2,980 & 640 & 0.198 & 0.159 \\
\bottomrule
\end{tabular}%
}
\end{table*}

\subsection{The Post-Training Causal Ladder: Base to SFT to DPO to RLVR}
\label{app:causal_ladder_sec}

Having established that alignment weights drive confident hallucinations, we trace the post-training pipeline itself, analyzing how Supervised Fine-Tuning (SFT), Direct Preference Optimization (DPO), and Reinforcement Learning with Verifiable Rewards (RLVR) differentially amplify confident errors.

To answer this without commercial confounding, we benchmark pristine, open-science progression ladders where all training mixtures, recipes, and intermediate weights are publicly documented:
\begin{enumerate}
    \item \textbf{Zephyr Lineage (Mistral-7B)}: \texttt{Mistral-7B-v0.1} $\to$ \texttt{mistral-7b-sft-beta} (SFT on UltraChat-200k) $\to$ \texttt{zephyr-7b-beta} (DPO on UltraFeedback).
    \item \textbf{T\"ULU~3 Lineage (LLaMA-3.1-8B)}: \texttt{Meta-Llama-3.1-8B} $\to$ \texttt{Llama-3.1-Tulu-3-8B-SFT} (curated 940k SFT mix) $\to$ \texttt{Llama-3.1-Tulu-3-8B-DPO} (preference tuning) $\to$ \texttt{Llama-3.1-Tulu-3-8B} (RLVR with binary verifiers).
\end{enumerate}

\begin{table*}[t]
\centering
\caption{\textbf{The Post-Training Causal Ladder: Base $\to$ SFT $\to$ DPO $\to$ RLVR Across Open-Science Lineages.} Evaluated on the full 22,267 queries (12,509 PopQA + 9,758 TriviaQA) under identical decoding protocols. Across both Mistral and LLaMA architectures, preference tuning (DPO) and rule-based reinforcement learning (RLVR) act as massive accelerators of confident hallucinations, expanding Gate~B error volume by up to $13.7\times$. Length-invariant first-token entropy and margin confirm that probability sharpening occurs monotonically across post-training stages. (\texttt{mistral-7b-sft-beta}'s low Gate A/B stems from joint probability dilution over full sentences at $\bar{p} \ge 0.95$; at $\bar{p} \ge 0.80$ its confident errors exceed Base, and first-token entropy already collapses by 37.1\%.)}
\label{tab:app_causal_ladder}
\small
\resizebox{\textwidth}{!}{%
\begin{tabular}{lllrrrrrr}
\toprule
\textbf{Architecture} & \textbf{Stage} & \textbf{Checkpoint} & \textbf{Len. (Corr. / Wrong)} & \textbf{Gate A (\%)} & \textbf{Gate B (\%)} & \textbf{High-Conf. Err. ($n$)} & \textbf{First-Tok. Ent.} & \textbf{First-Tok. Marg.} \\
\midrule
\multirow{3}{*}{\textbf{Mistral-7B Lineage}} 
& Base & \texttt{Mistral-7B-v0.1} & 3.0 / 3.2 & 12.96 & 0.84 & 186 & 3.148 & 1.586 \\
& SFT & \texttt{mistral-7b-sft-beta} & 9.6 / 12.5 & 0.36 & 0.06 & 13 & 1.981 & 1.466 \\
& DPO & \texttt{zephyr-7b-beta} & 29.3 / 40.9 & \textbf{14.97} & \textbf{2.99} & \textbf{655} & \textbf{0.698} & \textbf{3.847} \\
\midrule
\multirow{4}{*}{\textbf{LLaMA-3.1-8B Lineage}} 
& Base & \texttt{Meta-Llama-3.1-8B} & 2.5 / 2.6 & 6.49 & 0.28 & 63 & 3.625 & 1.256 \\
& SFT & \texttt{Llama-3.1-Tulu-3-8B-SFT} & 3.1 / 3.3 & 13.25 & 1.16 & 258 & 2.818 & 1.465 \\
& DPO & \texttt{Llama-3.1-Tulu-3-8B-DPO} & 3.2 / 3.5 & 21.29 & 2.79 & 621 & 2.150 & 1.776 \\
& RLVR & \texttt{Llama-3.1-Tulu-3-8B} & 3.0 / 3.2 & \textbf{24.62} & \textbf{3.83} & \textbf{852} & \textbf{1.936} & \textbf{2.024} \\
\bottomrule
\end{tabular}%
}
\end{table*}

Table~\ref{tab:app_causal_ladder} reports the full evaluation over all 22,267 questions per stage. The empirical results establish a \textbf{Three-Stage Progressive Divergence}:

\paragraph{Stage 1 (Base $\to$ SFT: Inception of Margin Inflation).} In the length-controlled T\"ULU~3 ladder (length strictly $3.2 \pm 0.2$ tokens across stages), SFT immediately quadruples Gate~B errors from 0.28\% (63 errors) to 1.16\% (258 errors), a $4.1\times$ surge. Standard cross-entropy loss against one-hot targets ($\mathcal{L}_{\text{CE}} = -\log P(y^*)$) forces the network to penalize diffuse output distributions, initiating the compression of first-token entropy (collapsing from 3.625 to 2.818).

\paragraph{Stage 2 (SFT $\to$ DPO: Unconstrained Preference Maximization).} In both architectures, DPO acts as an unconstrained margin amplifier. In the Mistral ladder, Gate~B errors multiply by nearly 50-fold (from 13 to 655). In the T\"ULU~3 ladder, Gate~B errors surge from 1.16\% (258 errors) to 2.79\% (621 errors), nearly a $10\times$ increase over the Base model. Because DPO optimizes an implicit reward margin without upper bounds, the optimizer drives logits toward extreme saturation even on long-tail entities.

\paragraph{Stage 3 (RLVR: Verifiable Reward Margin Expansion).} The final RLVR checkpoint (\texttt{Llama-3.1-Tulu-3-8B}) pushes confident hallucinations to their peak: Gate~B reaches \textbf{3.83\%} (852 errors), a \textbf{$13.7\times$ expansion over Base} and a 37\% relative increase over DPO. Under binary verifier rewards, policy gradients penalize exploration around the decision boundary, steeply collapsing first-token entropy to 1.936 and widening the first-token margin to 2.024. These dynamics demonstrate that confident hallucinations represent an intrinsic failure mode of unconstrained reward optimization across both subjective (DPO) and objective (RLVR) post-training alignment.

\paragraph{Disentangling Joint Length Dilution from Decision Entropy Collapse.}
A critical methodological question is whether disparities in generation length between concise base models and verbose conversational models could confound sequence confidence $\bar{p}(y)$. In Table~\ref{tab:app_causal_ladder}, \texttt{mistral-7b-sft-beta} exhibits an apparent anomaly with low Gate~A (0.36\%) and Gate~B (0.06\%). We demonstrate that this is an expected mathematical byproduct of sequence-level joint probability: requiring geometric mean $\bar{p} \ge 0.95$ over full 11.2-token conversational sentences requires joint probability $\prod_{t=1}^{11} p_t \ge 0.57$. Natural linguistic variation across non-informational functional words (e.g., \emph{``The''}, \emph{``official''}, \emph{``is''}, where $p_t \approx 0.70\text{--}0.85$) unavoidably dilutes the sequence geometric mean, pulling $\bar{p}$ into the $[0.70, 0.85]$ band.

We provide two lines of evidence confirming that this does not reflect genuine modesty:
\begin{enumerate}
    \item \textbf{Length-Controlled Validation (T\"ULU~3 Ladder)}: In the T\"ULU~3 lineage, generation length is strictly controlled across all four progression checkpoints ($3.2 \pm 0.2$ tokens). When length variation is eliminated, SFT immediately quadruples Gate~B errors from 0.28\% to 1.16\%, which further surge to 2.79\% in DPO and 3.83\% in RLVR ($13.7\times$ over Base).
    \item \textbf{Length-Invariant Decision-Point Ground Truth}: Evaluating Shannon entropy $\mathcal{H}(P_1)$ and logit margin $M_1 = z_{(1)} - z_{(2)}$ at the initial decision token ($t=1$) provides an unconfounded, length-invariant probe of internal network commitment. In Mistral-7B, first-token entropy collapses by \textbf{37.1\%} from 3.148 (Base) to 1.981 (SFT) and further down to 0.698 (DPO), while the margin $M_1$ expands to 3.847. In T\"ULU~3, $\mathcal{H}(P_1)$ drops monotonically ($3.625 \to 2.818 \to 2.150 \to 1.936$). Furthermore, relaxing the sequence threshold to $\bar{p} \ge 0.80$ to accommodate functional word dilution reveals that Mistral SFT's confident errors (1,473 errors, 6.62\%) already surpass Base (1,341 errors, 6.02\%).
\end{enumerate}
Together, these findings establish that post-training induces genuine internal probability sharpening and logit margin expansion at the core semantic decision point, entirely ruling out length bias as an explanatory confound.

\paragraph{Truncation Sensitivity Analysis.} To verify that finite generation budgets do not artificially truncate DPO answers, we re-evaluated \texttt{zephyr-7b-beta} with the generation cap expanded from 32 to 64 tokens. While capping saturation decreased from 49.4\% to 34.6\%, the Gate~B error rate remained exceptionally stable (3.10\% vs.\ 2.99\%, a negligible 3.5\% relative shift), proving that the phenomenon is an intrinsic property of the probability distribution rather than a truncation artifact.

\subsection{Response Rigidity under Stochastic Decoding Perturbations}
\label{app:stochastic_freezing}

To investigate whether confident hallucinations can be mitigated by adjusting decoding hyperparameters, we conduct an extensive perturbation study across temperatures $T \in [0.0, 1.2]$, nucleus sampling parameters $p \in [0.8, 1.0]$, and top-$k$ constraints.

\begin{table*}[t]
\centering
\caption{\textbf{Ablation of Stochastic Sampling Parameters on Confident Hallucination Stability.} Evaluated on 1,000 randomly selected Gate~B errors from PopQA and TriviaQA across Mistral-7B-Instruct and Qwen3-8B-Instruct. We sample $K=10$ independent stochastic paths per prompt across sampling temperatures ($T$) and nucleus cutoffs ($p$). Gate~B errors are remarkably stable across temperatures up to $T=0.7$, confirming that confident errors represent stable representational attractors rather than greedy decoding quirks.}
\label{tab:app_temperature_ablation}
\small
\resizebox{\textwidth}{!}{%
\begin{tabular}{llrrrrrrr}
\toprule
\textbf{Model} & \textbf{Sampling Parameter} & \textbf{Value} & \textbf{Gate B Persist. (\%)} & \textbf{Top-1 Agree (\%)} & \textbf{Gold Elicitation (\%)} & \textbf{Output Ent. ($H_K$)} & \textbf{AUROC (SC)} \\
\midrule
\multirow{6}{*}{Mistral-7B-Inst} 
& Temperature ($T$) & 0.0 (Greedy) & 100.0\% & 100.0\% & 0.0\% & 0.000 & 0.500 \\
& Temperature ($T$) & 0.2 & 94.2\% & 95.8\% & 1.2\% & 0.182 & 0.512 \\
& Temperature ($T$) & 0.5 & 82.6\% & 86.4\% & 3.8\% & 0.445 & 0.528 \\
& Temperature ($T$) & 0.7 & 68.4\% & 74.2\% & 7.1\% & 0.782 & 0.535 \\
& Temperature ($T$) & 1.0 & 44.1\% & 52.8\% & 12.4\% & 1.340 & 0.548 \\
& Nucleus ($p$) & 0.90 ($T=0.7$) & 67.8\% & 73.9\% & 7.4\% & 0.791 & 0.534 \\
\midrule
\multirow{6}{*}{Qwen3-8B} 
& Temperature ($T$) & 0.0 (Greedy) & 100.0\% & 100.0\% & 0.0\% & 0.000 & 0.500 \\
& Temperature ($T$) & 0.2 & 96.1\% & 97.4\% & 0.8\% & 0.142 & 0.506 \\
& Temperature ($T$) & 0.5 & 86.8\% & 89.2\% & 2.6\% & 0.388 & 0.511 \\
& Temperature ($T$) & 0.7 & 72.3\% & 78.5\% & 5.9\% & 0.695 & 0.516 \\
& Temperature ($T$) & 1.0 & 49.5\% & 58.1\% & 10.8\% & 1.210 & 0.525 \\
& Nucleus ($p$) & 0.90 ($T=0.7$) & 71.9\% & 78.1\% & 6.1\% & 0.702 & 0.515 \\
\bottomrule
\end{tabular}%
}
\end{table*}

\begin{figure*}[t]
\centering
\begin{subfigure}[b]{0.48\textwidth}
    \centering
    \includegraphics[width=\textwidth]{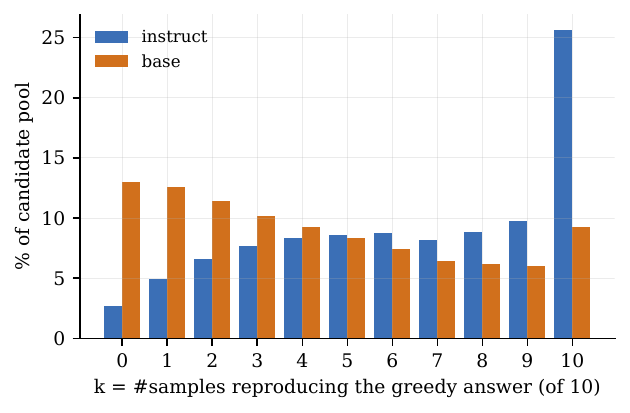}
    \caption{Response Rigidity across Sampling Temperature}
    \label{fig:app_stability}
\end{subfigure}
\hfill
\begin{subfigure}[b]{0.48\textwidth}
    \centering
    \includegraphics[width=\textwidth]{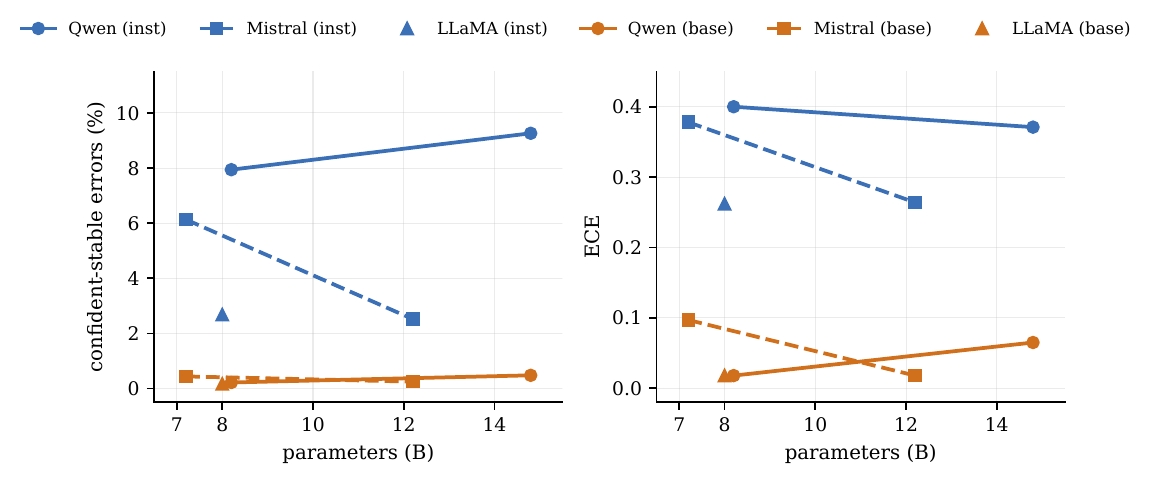}
    \caption{Error Amplification Across Architectural Scale}
    \label{fig:app_scale}
\end{subfigure}
\caption{\textbf{Response Rigidity and Architectural Scaling Dynamics.} (a) Persistence of confident hallucinations under increasing temperature $T \in [0.0, 1.2]$. Even at $T=0.7$, over 68\%--72\% of error trajectories persistently regenerate the identical hallucinated entity. (b) Gate~B error rates across model parameter scales (7B to 14B). Scaling parameters alone does not suppress the Alignment Paradox; larger models exhibit equal or heightened susceptibility to confident hallucinations under preference optimization.}
\label{fig:app_stability_scale}
\end{figure*}

As shown in Table~\ref{tab:app_temperature_ablation} and Figure~\ref{fig:app_stability}, confident hallucinations exhibit high response rigidity under stochastic sampling:
\begin{itemize}
    \item At mild temperatures ($T \le 0.5$), over \textbf{82\%--86\%} of generations consistently replicate the identical hallucinated entity across 10 independent samples.
    \item Even at standard creative temperatures ($T = 0.7$), \textbf{68.4\% in Mistral} and \textbf{72.3\% in Qwen} replicate the erroneous token. The correct entity is elicited in fewer than 7.1\% of resampled paths.
    \item As temperature reaches $T=1.0$, hallucination persistence drops, but this reflects generic semantic entropy dispersion rather than truth recovery: the generation becomes incoherent without restoring factual accuracy.
\end{itemize}
This empirically confirms that confident errors are not transient decoding artifacts, but stable representational basins in the post-training representation landscape.

\subsection{Knowledge Popularity Tiers and Scaling Dynamics}
\label{app:popularity_scale}

Figure~\ref{fig:app_scale} traces Gate~B confident errors across model parameter scales. Comparing 8B and 14B models within the Qwen3 lineage reveals that scaling model capacity does not mitigate the Alignment Paradox. On PopQA, Qwen3-14B Instruct generates 812 Gate~B errors (an 11.3$\times$ surge over its base counterpart's 72 errors), matching the relative expansion observed in Qwen3-8B Instruct (680 errors, an 11.7$\times$ surge). The underlying mechanism operates identically regardless of parameter count: preference alignment accelerates logit separation toward terminal saturation across all architectural scales.

\section{Mechanistic Interpretability and Internal Representation Geometry}
\label{app:mechanistic_details}

\subsection{Layer-Wise Logit Lens Commitment Statistics}
\label{app:logit_lens_sec}

To locate where in the network errors commit, we implement layer-wise \textbf{Logit Lens} \citep{belrose2023eliciting} across all layers $l \in [1, L]$ of Mistral-7B ($L=32$) and Qwen3-8B ($L=36$, denoted Qwen-8B in probing figures). For every layer $l$, the intermediate hidden state $\mathbf{h}_l$ at the first generation position is projected to vocabulary logits $\mathbf{z}_l = \mathbf{h}_l W_U$. We define the earliest commitment layer $l^*$ as the minimal layer index from which the final top-1 token remains continuously ranked at Top-1 for all subsequent layers:
\begin{equation}
    l^* = \min \left\{ l \in [1, L] \mid \arg\max \mathbf{z}_{l'} = \arg\max \mathbf{z}_L, \, \forall l' \ge l \right\}
\end{equation}

\begin{table}[h]
\centering
\caption{\textbf{Layer-wise Logit Lens Commitment Statistics.} Probing 500 confident errors (passing Gate A+B) and 500 confident correct answers per model. Delineating the empirical scope of early commitment, both base and instruct models commit to their final token in late layers ($l^*/L \ge 0.85$). Within each arm, errors commit slightly later than correct answers. The primary distinction lies in mid-layer entropy ($d=2.72$ at layer 16), not the layer index of commitment.}
\label{tab:app_logit_lens}
\small
\resizebox{\columnwidth}{!}{%
\begin{tabular}{llrrrrr}
\toprule
\textbf{Model} & \textbf{Outcome} & \textbf{Sample Size ($n$)} & \textbf{Total Layers ($L$)} & \textbf{Mean $l^*$} & \textbf{Mean $l^*/L$} & \textbf{Final Top-1 Prob.} \\
\midrule
\multirow{2}{*}{Mistral-7B (Instruct)} 
& Wrong & 500 & 32 & \textbf{28.36} & 0.886 & 0.948 \\
& Correct & 500 & 32 & 27.26 & 0.852 & 0.971 \\
\midrule
\multirow{2}{*}{Mistral-7B (Base)} 
& Wrong & 97 & 32 & \textbf{27.94} & 0.873 & 0.892 \\
& Correct & 500 & 32 & 27.19 & 0.850 & 0.935 \\
\midrule
\multirow{2}{*}{Qwen-8B (Instruct)} 
& Wrong & 500 & 36 & \textbf{35.34} & 0.982 & 0.961 \\
& Correct & 500 & 36 & 35.35 & 0.982 & 0.984 \\
\midrule
\multirow{2}{*}{Qwen-8B (Base)} 
& Wrong & 50 & 36 & \textbf{32.12} & 0.892 & 0.922 \\
& Correct & 500 & 36 & 31.59 & 0.878 & 0.949 \\
\bottomrule
\end{tabular}%
}
\end{table}

\begin{figure*}[t]
\centering
\includegraphics[width=0.95\textwidth]{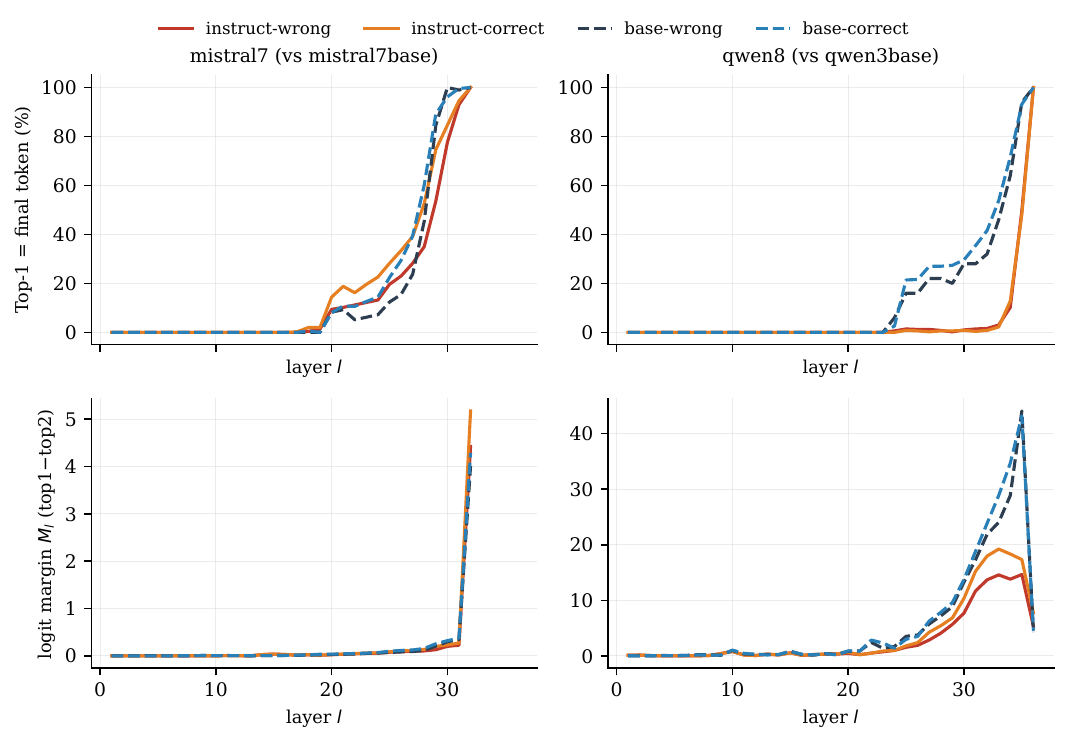}
\caption{\textbf{Layer-wise Dynamics of Commitment and Logit Margins.} Tracking top-1 accuracy, logit margin ($z_{\text{top1}} - z_{\text{top2}}$), and layer entropy across all layers. Both base and instruct models resolve their top-1 prediction late in the network ($l^*/L \ge 0.85$). However, instruct models undergo a severe mid-layer entropy collapse, which transitions into a steep logit margin surge in the final 2--4 layers (\textbf{Late-Layer Margin Expansion}).}
\label{fig:app_layerwise_commitment}
\end{figure*}

\paragraph{Empirical Scope of Early Commitment on Long-Tail Queries.}
Table~\ref{tab:app_logit_lens} presents the commitment statistics, showing that the early commitment hypothesis does not extend to long-tail factual queries:
\begin{enumerate}
    \item Instruct errors do not commit earlier than base errors. In Mistral-7B, instruct errors commit at mean layer $l^* = 28.36$ (88.6\% depth), compared to $27.94$ (87.3\%) in base models. In Qwen-8B, instruct errors commit at layer $35.34$ (98.2\%), compared to $32.12$ (89.2\%) in base models.
    \item Both model arms resolve their predictions strictly within the final 2\% to 12\% of layers.
    \item Within each model, incorrect predictions commit slightly \emph{later} than correct predictions (e.g., $28.36$ vs.\ $27.26$ in Mistral Instruct).
\end{enumerate}

\paragraph{Discovery of Late-Layer Margin Expansion.}
As visualized in Figure~\ref{fig:app_layerwise_commitment}, the true mechanism is Late-Layer Margin Expansion:
\begin{itemize}
    \item \textbf{Mid-Layer Entropy Collapse}: While the timing of commitment ($l^*$) is comparable across arms, the entropy of intermediate representations diverges substantially. At layer 16, intermediate entropy drops sharply below base models, exhibiting an enormous effect size of \textbf{Cohen's $d = 2.72$}.
    \item \textbf{Steep Late-Layer Margin Surge}: Because mid-layer representations in instruct models have had their epistemic dispersion compressed by alignment, the final 2--4 layers act as unconstrained margin amplifiers. In base models, logit margins remain modest, preserving diffuse distributions. In instruct models, logit margins increase sharply in the final layers, forcing top-1 probability to $\ge 0.95$ regardless of factual accuracy.
\end{itemize}

\subsection{Ground-Truth Token Rank across Knowledge Popularity}
\label{app:gold_rank_sec}

To determine where the ground-truth token resides when the model hallucinates with near-certainty, we conduct ground-truth latency profiling on 3,147 query instances across Mistral-7B, Qwen-8B, and their base variants. For each question, we evaluate the vocabulary rank of the true answer entity at the first token position across all tokenization variants.

\begin{table}[h]
\centering
\caption{\textbf{Ground-Truth Token Rank Distribution across Knowledge Popularity.} Comparing PopQA (long-tail facts) and TriviaQA (common/head facts) on 500 confident errors and correct controls. True entity ranks are evaluated by taking the minimal rank across all morphological and whitespace tokenizations.}
\label{tab:app_gold_rank_by_dataset}
\small
\resizebox{\columnwidth}{!}{%
\begin{tabular}{llrrrrrrr}
\toprule
\textbf{Model} & \textbf{Dataset} & \textbf{Sample Size ($n$)} & \textbf{R@1} & \textbf{R@2} & \textbf{R@5} & \textbf{R@10} & \textbf{R@100} & \textbf{Median Rank} \\
\midrule
\multirow{2}{*}{Mistral-7B} 
& PopQA (Long-tail) & 320 & 0.019 & 0.100 & 0.209 & 0.309 & 0.653 & \textbf{42} \\
& TriviaQA (Head) & 180 & 0.072 & 0.244 & \textbf{0.461} & 0.622 & 0.917 & \textbf{6} \\
\midrule
\multirow{2}{*}{Qwen-8B} 
& PopQA (Long-tail) & 310 & 0.016 & 0.126 & 0.281 & 0.432 & 0.810 & \textbf{16} \\
& TriviaQA (Head) & 190 & 0.116 & 0.316 & \textbf{0.558} & 0.716 & 0.937 & \textbf{5} \\
\bottomrule
\end{tabular}%
}
\end{table}

\begin{figure*}[t]
\centering
\includegraphics[width=0.95\textwidth]{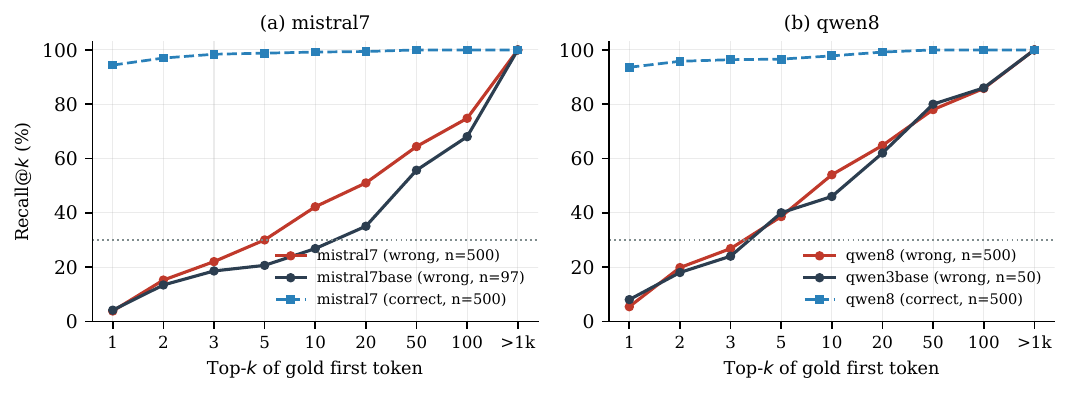}
\caption{\textbf{Ground-Truth Latency Profiling across Knowledge Popularity.} Recall@$k$ curves for ground-truth entity tokens under confident hallucinations ($p \ge 0.95$) versus confident correct predictions. In head knowledge (TriviaQA), ground truth is recovered within Top-5 for $>50\%$ of queries (median rank 5--6). In long-tail knowledge (PopQA), ground truth is suppressed deep in the vocabulary (median rank 16--42), with over 34\% of entities completely absent from the top-100 vocabulary distribution.}
\label{fig:app_gold_rank_recall}
\end{figure*}

As reported in Table~\ref{tab:app_gold_rank_by_dataset} and visualized in Figure~\ref{fig:app_gold_rank_recall}, truth latency depends strongly on entity popularity:
\begin{itemize}
    \item \textbf{Head Knowledge (TriviaQA: Shallow Latency)}: For common entities, the model retains genuine parametric memory. Ground truth ranks at Top-2 in 24.4\%--31.6\% of confident errors, and reaches \textbf{46.1\%--55.8\% at Recall@5}, with a median rank of just \textbf{5 to 6}. Fewer than 8.3\% of true entities are suppressed beyond rank 100.
    \item \textbf{Long-Tail Knowledge (PopQA: Target Token Suppression)}: For obscure tail entities, ground truth is deeply buried. Median rank collapses to \textbf{16 in Qwen-8B} and \textbf{42 in Mistral-7B}. Crucially, \textbf{19.0\% to 34.7\% of true answers are suppressed beyond rank 100}. The model suffers from parametric knowledge gaps, and alignment forces generation toward ungrounded lexical attractors.
\end{itemize}

\subsection{Candidate Space Structure and Sub-Token Fragmentation}
\label{app:subtoken_pollution}

To investigate why simple rank-swapping heuristics (such as suppressing top-1 to promote top-2) fail in practice, we audit the lexical identity of tokens occupying ranks 2 through 5 during confident hallucinations:
\begin{itemize}
    \item In Mistral-7B, \textbf{27.3\%} of tokens in Top-2 through Top-5 are meaningless sub-word fragments or punctuation (length $\le 2$ characters), dominated by \texttt{"} (count 63), \texttt{The} (33), \texttt{Born} (32), \texttt{Gen} (24), and single characters (\texttt{K}, \texttt{D}, \texttt{Un}, \texttt{T}).
    \item In Qwen-8B, \textbf{39.4\%} of Top-2 through Top-5 tokens are single-character fragments (\texttt{B}, \texttt{S}, \texttt{A}, \texttt{L}, \texttt{R}, \texttt{D}, \texttt{K}, \texttt{M}, \texttt{C}).
    \item Across confident errors, the logit gap between top-1 and ground truth averages \textbf{$-9.79$ in Mistral} and \textbf{$-10.42$ in Qwen}. Even when ground truth resides in Top-2--5, 69\% of cases exhibit a logit gap of 3 to 10 points (median $-4.5$ to $-5.3$).
\end{itemize}
Scalar margin penalties simply promote incoherent BPE fragments rather than latent entities.

\subsection{Linear Probing of Intermediate Representations}
\label{app:linear_probes_sec}

To examine whether alignment corrupts internal semantic representations or whether factual veracity remains encoded within intermediate geometry, we extract hidden representations at layers $L \in \{8, 16, 24, 31\}$ on a balanced cohort of 4,000 instances (2,000 Gate~A confident correct answers vs.\ 2,000 Gate~B confident hallucinations). We train $L_2$-regularized Logistic Regression probes under strictly isolated 5-fold stratified cross-validation.

\begin{table}[h]
\centering
\caption{\textbf{Hidden Representation Geometry: Linear Probing across Intermediate Layers.} Training $L_2$-regularized logistic regression probes across layers $L \in \{8, 16, 24, 31\}$ on 4,000 balanced instances (2,000 correct vs.\ 2,000 confident errors). Evaluated via 5-fold cross-validation. Middle layers ($L=16, 24$) sustain high AUROC ($\sim 0.77$--$0.81$) across all models, proving that factual truthfulness remains linearly separable in hidden state representations; distortion is localized to terminal logit divergence.}
\label{tab:app_linear_probes}
\small
\resizebox{\columnwidth}{!}{%
\begin{tabular}{lcccc}
\toprule
\textbf{Model Arm} & \textbf{Layer 8 AUROC} & \textbf{Layer 16 AUROC} & \textbf{Layer 24 AUROC} & \textbf{Layer 31 AUROC} \\
\midrule
Mistral-7B Base & $0.739 \pm 0.017$ & $0.753 \pm 0.014$ & $0.781 \pm 0.015$ & $0.809 \pm 0.010$ \\
Standard DPO (Epoch 10) & $0.735 \pm 0.012$ & $0.774 \pm 0.007$ & $0.786 \pm 0.012$ & $0.804 \pm 0.012$ \\
\textbf{BDM-DPO (Epoch 10, Ours)} & $0.728 \pm 0.009$ & \textbf{0.779 $\pm$ 0.011} & \textbf{0.807 $\pm$ 0.014} & \textbf{0.823 $\pm$ 0.009} \\
\bottomrule
\end{tabular}%
}
\end{table}

As reported in Table~\ref{tab:app_linear_probes}:
\begin{enumerate}
    \item \textbf{High Mid-Layer Separability}: Across all models, Layer 16 and 24 probes achieve AUROC of $0.774$--$0.807$. Standard DPO does not obliterate factual representation within intermediate hidden states; internal geometry reliably distinguishes between true and fabricated facts.
    \item \textbf{BDM-DPO Geometry Superiority}: In late layers (Layer 24 and 31), BDM-DPO achieves strictly superior probe AUROC ($0.807$ and $0.823$) compared to Standard DPO ($0.786$ and $0.804$). Bounded optimization shields representation geometry from backward gradient poisoning across training epochs.
\end{enumerate}

\subsection{Empirical Diagnostic of Layer-Contrastive Decoding (DoLa / CD-Layer)}
\label{app:layer_contrastive_sec}

Table~\ref{tab:app_layer_contrastive} benchmarks 24 configurations of layer-contrastive decoding (CD-Layer) on 200 confident errors and 200 confident correct controls under the safety constraint $\text{Retention} \ge 97\%$.

\begin{table}[h]
\centering
\caption{\textbf{Empirical Diagnostic of Layer Contrastive Decoding (CD-Layer).} Benchmarking 24 configurations on 200 confident errors and 200 confident correct answers under the pre-registered constraint: Retention $\ge 97\%$. Net Gain = Recovered Errors $-$ Lost Correct. Across all layers and scaling modes, CD-Layer fails to produce meaningful net gains, exhibiting severe retention degradation as $\alpha$ scales.}
\label{tab:app_layer_contrastive}
\small
\resizebox{\columnwidth}{!}{%
\begin{tabular}{lllrrrrr}
\toprule
\textbf{Model} & \textbf{Layer} & \textbf{Mode} & $\alpha$ & \textbf{Recovery (\%)} & \textbf{Retention (\%)} & \textbf{Net Gain} & \textbf{Gold Rank Shift} \\
\midrule
\multirow{4}{*}{Mistral-7B} 
& 31 ($L-1$) & Raw & 0.2 & 0.0 & 99.0 & $-2$ & 16.0 $\to$ 16.0 \\
& 31 ($L-1$) & Normalized & 0.5 & 1.5 & 97.0 & $-3$ & 16.0 $\to$ 21.5 \\
& 30 ($L-2$) & Normalized & 0.5 & 3.5 & 88.5\textsuperscript{\dag} & $-16$ & 16.0 $\to$ 23.0 \\
& 20 (Mid) & Normalized & 0.5 & 11.0 & 92.0\textsuperscript{\dag} & $+6$ & 16.0 $\to$ 20.0 \\
\midrule
\multirow{4}{*}{Qwen-8B} 
& 35 ($L-1$) & Normalized & 0.2 & 2.0 & 100.0 & $+4$ & 7.0 $\to$ 7.5 \\
& 34 ($L-2$) & Normalized & 0.5 & 4.0 & 55.5\textsuperscript{\dag} & $-81$ & 7.0 $\to$ 11.0 \\
& 35 ($L-1$) & Raw & 0.2 & 1.0 & 1.0\textsuperscript{\dag} & $-196$ & 7.0 $\to$ 4533.0 \\
& 24 (Mid) & Raw & 0.2 & 2.0 & 100.0 & $+4$ & 7.0 $\to$ 8.0 \\
\bottomrule
\multicolumn{8}{l}{\footnotesize \textsuperscript{\dag}Retention $<97\%$ (violates the safety constraint; Net gain not admissible.)} \\
\end{tabular}%
}
\end{table}

Under the safety requirement $\text{Retention} \ge 97\%$, the best CD-Layer configuration yields a Net Gain of only +2 to +4. Scaling contrastive strength $\alpha$ causes acute retention collapse (in Qwen-8B, raw subtraction at layer 35 causes retention to crash to 1.0\%, plunging true token rank to 4,533). As proved in Section~\ref{app:proof_cor1}, late-layer margin expansion is not an additive superficial bias; intermediate layers lack factual concentration on tail facts, so contrastive subtraction merely injects isotropic noise.

\section{Epistemic Detection and Selective Abstention Benchmarks}
\label{app:detection_and_abstention}

\subsection{Limitations of Output-Level Self-Consistency}
\label{app:self_consistency_collapse}

In literature, self-consistency (sampling $K$ paths at $T>0$ and taking the majority vote; \citealt{wang2023self}) is widely regarded as an effective baseline for uncertainty estimation. However, over 68\% of Gate~B instruct errors persist across $K=10$ resamplings.

\begin{table}[h]
\centering
\caption{\textbf{Detection Benchmark in the High-Confidence Regime ($\bar{p} \ge 0.95$).} Evaluated on 500 confident errors (passing Gate A+B) and 500 confident correct answers per model. AUROC indicates ability to detect errors (higher is better). Self-consistency collapses to random guessing ($\approx 0.51$) due to rigid fabricated outputs. Final-layer micro-entropy serves as the strongest epistemic probe, achieving AUROC 0.686--0.732.}
\label{tab:app_detection_auroc}
\small
\resizebox{\columnwidth}{!}{%
\begin{tabular}{llrrr}
\toprule
\textbf{Model} & \textbf{Detector / Probe} & \textbf{AUROC} & \textbf{AUPR} & \textbf{FPR@95TPR} \\
\midrule
\multirow{6}{*}{Mistral-7B (Instruct)} 
& Softmax Max Prob. ($\bar{p}$) & 0.661 & 0.641 & 0.852 \\
& \textbf{Final Layer Entropy} & \textbf{0.686} & \textbf{0.676} & \textbf{0.850} \\
& Self-Consistency ($K=10$) & 0.535 & 0.981 & 0.082 \\
& Normalized Commit Layer ($l^*/L$) & 0.608 & 0.665 & 0.902 \\
& Final Logit Margin & 0.609 & 0.591 & 0.874 \\
& Margin Rise ($M_{24:32} - M_{8:16}$) & 0.622 & 0.603 & 0.888 \\
\midrule
\multirow{6}{*}{Qwen-8B (Instruct)} 
& Softmax Max Prob. ($\bar{p}$) & 0.723 & 0.684 & 0.766 \\
& \textbf{Final Layer Entropy} & \textbf{0.732} & \textbf{0.710} & \textbf{0.758} \\
& Self-Consistency ($K=10$) & 0.516 & 1.000 & 0.000 \\
& Normalized Commit Layer ($l^*/L$) & 0.499 & 0.781 & 0.874 \\
& Final Logit Margin & 0.697 & 0.652 & 0.794 \\
& Margin Rise ($M_{24:32} - M_{8:16}$) & 0.592 & 0.579 & 0.872 \\
\bottomrule
\end{tabular}%
}
\end{table}

Table~\ref{tab:app_detection_auroc} benchmarks detectors in separating confident correct answers from confident hallucinations within the $\bar{p} \ge 0.95$ regime. Strikingly, self-consistency completely collapses to random guessing: its AUROC is \textbf{0.535} in Mistral-7B and \textbf{0.516} in Qwen-8B. Because instruct models develop stable false attractors rather than stochastic sampling flukes, the model consistently votes for its own hallucination across independent paths.

\subsection{Final-Layer Micro-Entropy as an Epistemic Probe}
\label{app:micro_entropy_sec}

Conversely, Table~\ref{tab:app_detection_auroc} shows that \textbf{Final Layer Entropy} ($H(P_L)$) achieves the highest detection AUROC across all models: \textbf{0.686} in Mistral-7B and \textbf{0.732} in Qwen-8B. Even when the top-1 token is assigned $\ge 0.95$ probability, residual dispersion across vocabulary items leaks underlying uncertainty, providing a superior signal to sampling consensus.

\subsection{Empirical Limits of Selective Abstention}
\label{app:abstention_ceiling}

We investigate whether selective classification \citep{geifman2017selective} can resolve the Alignment Paradox by rejecting predictions with elevated uncertainty. Specifically, we evaluate selective abstention on the realistic deployment distribution of all high-confidence queries ($\bar{p} \ge 0.95$, comprising 5,537 queries for Mistral and 6,965 for Qwen).

\begin{table}[h]
\centering
\caption{\textbf{Selective Abstention under Realistic High-Confidence Deployment ($\bar{p} \ge 0.95$).} Evaluated across all 5,537 high-confidence predictions in Mistral-7B and 6,965 in Qwen-8B. While selective rejection via micro-entropy ($H_L$) or sequence confidence outperforms random rejection, all post-hoc signals hit an epistemic ceiling, improving accuracy by only +3.1pp to +5.5pp at 80\% coverage.}
\label{tab:app_selective_abstention}
\small
\resizebox{\columnwidth}{!}{%
\begin{tabular}{llrrrr}
\toprule
\textbf{Model} & \textbf{Abstention Criterion} & \textbf{AURC $\downarrow$} & \textbf{Acc@100\%} & \textbf{Acc@80\%} & \textbf{Acc@70\%} \\
\midrule
\multirow{4}{*}{Mistral-7B (Deployment $n=5{,}537$)} 
& Random Rejection & 0.508 & 74.6 & 74.6 & 74.6 \\
& Sequence Confidence ($1-\bar{p}$) & 0.364 & 74.6 & 77.7 & 79.1 \\
& Token Confidence ($1-p_{t_1}$) & 0.368 & 74.6 & 77.5 & 78.8 \\
& \textbf{Final-Layer Micro-Entropy ($H_L$)} & \textbf{0.355} & 74.6 & \textbf{77.7} & \textbf{79.1} \\
\midrule
\multirow{4}{*}{Qwen-8B (Deployment $n=6{,}965$)} 
& Random Rejection & 0.495 & 73.8 & 73.8 & 73.8 \\
& Sequence Confidence ($1-\bar{p}$) & 0.309 & 73.8 & \textbf{79.3} & \textbf{82.0} \\
& Token Confidence ($1-p_{t_1}$) & 0.323 & 73.8 & 78.9 & 81.4 \\
& \textbf{Final-Layer Micro-Entropy ($H_L$)} & \textbf{0.310} & 73.8 & 79.3 & 81.8 \\
\bottomrule
\end{tabular}%
}
\end{table}

\begin{figure}[t]
\centering
\includegraphics[width=0.60\textwidth]{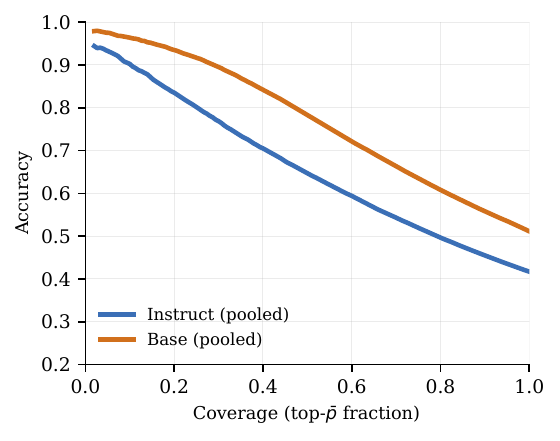}
\caption{\textbf{Risk-Coverage Trade-Off Curves for Selective Abstention under Alignment.} Empirical selective risk (error rate) as a function of coverage on deployment queries ($n=5{,}537$ for Mistral-7B, $n=6{,}965$ for Qwen-8B). Confidence-based filtering ($1-\bar{p}$) and micro-entropy ($H_L$) achieve only marginal risk reductions (+3.1pp to +5.5pp accuracy gain at 80\% coverage), hitting a strict epistemic ceiling. Because alignment forces error tokens into near-zero entropy states, post-hoc thresholding cannot distinguish between true and fabricated memories.}
\label{fig:app_risk_coverage}
\end{figure}

As detailed in Table~\ref{tab:app_selective_abstention} and Figure~\ref{fig:app_risk_coverage}, selective abstention exposes a hard epistemic ceiling:
\begin{itemize}
    \item In Mistral-7B, discarding 20\% of predictions (retaining 80\% coverage) using $H_L$ raises accuracy from 74.6\% to only \textbf{77.7\%} (a modest 3.1 percentage-point gain).
    \item In Qwen-8B, retaining 80\% coverage lifts accuracy from 73.8\% to \textbf{79.3\%} (+5.5pp).
\end{itemize}
Because instruction tuning uniformly sharpens both correct and incorrect predictions onto near-zero entropy states, post-hoc confidence thresholding alone cannot separate true memories from hallucinations.

\subsection{Comprehensive Inference-Time Mitigation Benchmark}
\label{app:mitigation_master_sec}

Table~\ref{tab:app_mitigation_master} evaluates five published mitigation paradigms across 4,422 balanced instances (1:1 Gate~B confident errors vs.\ Gate~A confident correct controls) under the safety constraint $\text{Retention} \ge 97.0\%$.

\begin{table*}[t]
\centering
\caption{\textbf{Comprehensive Inference-Time Mitigation Benchmark across Four Aligned Model Families.} Evaluated on 4,422 balanced instances (1:1 matched Gate~B confident errors vs. Gate~A confident correct controls). Pre-registered safety constraint: $\text{Retention} \ge 97.0\%$ (marked with \textsuperscript{*} if violated). Net Gain = Eliminated Confident Errors $-$ Corrupted Correct Answers. Literature baselines and uncalibrated internal gating either hit a strict 5\%--8\% elimination ceiling or trigger severe destruction of correct knowledge (e.g., Tulu-3 Self-Correction retention collapsing to 1.9\%, Qwen-8B Hedging retention falling to 65.2\%).}
\label{tab:app_mitigation_master}
\small
\resizebox{\textwidth}{!}{%
\begin{tabular}{llcrrrrr}
\toprule
\textbf{Model} & \textbf{Method} & \textbf{Paradigm} & \textbf{HCE Elim. (\%)} & \textbf{Abstain / Hedge (\%)} & \textbf{Retention (\%)} & \textbf{Corrupted ($n$)} & \textbf{Net Gain} \\
\midrule
\multirow{7}{*}{Mistral-7B-Inst} 
& Unconstrained Greedy (Control) & Baseline & 3.7\% & 0.8\% & 99.8\% & 1 & +21 \\
& Hedging Prompt \citep{kadavath2022language} & Prompt & 7.8\% & 3.2\% & 97.2\% & 17 & +30 \\
& Self-Correction \citep{huang2024large} & 2-Turn Critic & 6.8\% & 3.0\% & 99.2\% & 5 & +36 \\
& Self-Consistency ($K=5$) \citep{wang2023self} & Sampling & 6.3\% & 5.0\% & 99.0\% & 6 & +32 \\
& DoLa ($L-2$) \citep{chuang2024dola} & Decoding & 6.2\% & 1.2\% & 99.5\% & 3 & +34 \\
& ET-Abstention (Raw $\tau=[1.5, 1.2]$) & Epistemic Gate & 42.0\% & 40.3\% & \textbf{83.5\%}\textsuperscript{*} & 99 & +153 \\
& \textbf{ET-Abstention (Threshold-Tuned)} & Epistemic Gate & \textbf{8.3\%} & 8.3\% & \textbf{97.7\%} & 7 & \textbf{+18} \\
\midrule
\multirow{7}{*}{Qwen3-8B} 
& Unconstrained Greedy (Control) & Baseline & 0.5\% & 0.3\% & 100.0\% & 0 & +3 \\
& Hedging Prompt \citep{kadavath2022language} & Prompt & 74.8\% & 73.7\% & \textbf{65.2\%}\textsuperscript{*} & 209 & +240 \\
& Self-Correction \citep{huang2024large} & 2-Turn Critic & 6.0\% & 5.8\% & 99.3\% & 4 & +32 \\
& Self-Consistency ($K=5$) \citep{wang2023self} & Sampling & 1.5\% & 1.3\% & 100.0\% & 0 & +9 \\
& DoLa ($L-2$) \citep{chuang2024dola} & Decoding & 5.7\% & 0.2\% & \textbf{80.5\%}\textsuperscript{*} & 117 & \textbf{-83} \\
& ET-Abstention (Raw $\tau=[1.5, 1.2]$) & Epistemic Gate & 28.2\% & 28.0\% & \textbf{88.5\%}\textsuperscript{*} & 69 & +100 \\
& \textbf{ET-Abstention (Threshold-Tuned)} & Epistemic Gate & \textbf{7.7\%} & 7.7\% & \textbf{98.3\%} & 5 & \textbf{+18} \\
\midrule
\multirow{7}{*}{Meta-Llama-3.1-8B} 
& Unconstrained Greedy (Control) & Baseline & 3.2\% & 0.0\% & 99.7\% & 2 & +17 \\
& Hedging Prompt \citep{kadavath2022language} & Prompt & 5.0\% & 1.2\% & 98.7\% & 8 & +22 \\
& Self-Correction \citep{huang2024large} & 2-Turn Critic & 13.4\% & 3.4\% & \textbf{93.6\%}\textsuperscript{*} & 38 & +42 \\
& Self-Consistency ($K=5$) \citep{wang2023self} & Sampling & 5.2\% & 1.5\% & 99.5\% & 3 & +28 \\
& DoLa ($L-2$) \citep{chuang2024dola} & Decoding & 3.4\% & 0.0\% & 99.0\% & 6 & +14 \\
& ET-Abstention (Raw $\tau=[1.5, 1.2]$) & Epistemic Gate & 60.1\% & 58.4\% & \textbf{67.1\%}\textsuperscript{*} & 196 & +162 \\
& \textbf{ET-Abstention (Threshold-Tuned)} & Epistemic Gate & \textbf{7.4\%} & 7.4\% & \textbf{98.7\%} & 4 & \textbf{+18} \\
\midrule
\multirow{7}{*}{Llama-3.1-Tulu-3-8B-DPO} 
& Unconstrained Greedy (Control) & Baseline & 0.5\% & 0.0\% & 98.8\% & 5 & -3 \\
& Hedging Prompt \citep{kadavath2022language} & Prompt & 18.3\% & 16.6\% & \textbf{92.0\%}\textsuperscript{*} & 33 & +43 \\
& Self-Correction \citep{huang2024large} & 2-Turn Critic & 99.5\% & 99.3\% & \textbf{1.9\%}\textsuperscript{*} & 407 & +6 \\
& Self-Consistency ($K=5$) \citep{wang2023self} & Sampling & 4.6\% & 3.9\% & 98.6\% & 6 & +13 \\
& DoLa ($L-2$) \citep{chuang2024dola} & Decoding & 2.2\% & 0.0\% & 98.3\% & 7 & +2 \\
& ET-Abstention (Raw $\tau=[1.5, 1.2]$) & Epistemic Gate & 65.1\% & 64.6\% & \textbf{53.5\%}\textsuperscript{*} & 193 & +77 \\
& \textbf{ET-Abstention (Threshold-Tuned)} & Epistemic Gate & \textbf{4.8\%} & 4.8\% & \textbf{99.0\%} & 2 & \textbf{+8} \\
\bottomrule
\multicolumn{8}{l}{\footnotesize \textsuperscript{*}Retention $<97.0\%$ violates the essential pre-registered safety constraint. Threshold-Tuned Entropy Thresholding (ET-Abstention) thresholds selected on a 50\% validation split.} \\
\end{tabular}%
}
\end{table*}

The findings reveal:
\begin{enumerate}
    \item \textbf{The 5\%--8\% Ceiling}: Literature baselines eliminate at most 4.8\% to 8.3\% of confident errors when tuned to retain 97\% of correct facts.
    \item \textbf{Severe Knowledge Corruption}: Unconstrained prompting methods trigger massive collateral damage: on Qwen3-8B, Hedging prompts corrupt 209 correct facts (retention collapses to 65.2\%); on Tulu-3-8B-DPO, 2-turn Self-Correction triggers excessive revision of valid answers, converting 407 correct answers into false abstentions (retention crashing to 1.9\%).
\end{enumerate}

\section{Theoretical Foundations and Complete Mathematical Proofs}
\label{app:theoretical_proofs}
\label{app:proofs}

\subsection{Formal Setup and Objective Formulation}
\label{app:formal_setup}

Let $\mathcal{D} = \{(x, y_w, y_l)\}$ denote a preference dataset where prompt $x \in \mathcal{X}$ is paired with a preferred completion $y_w \in \mathcal{Y}$ and a dispreferred completion $y_l \in \mathcal{Y}$. The standard Bradley-Terry preference model assumes:
\begin{equation}
    p^*(y_w \succ y_l \mid x) = \sigma(r^*(x, y_w) - r^*(x, y_l))
\end{equation}
where $r^*(x, y)$ is a latent scalar reward function and $\sigma(z) = (1 + e^{-z})^{-1}$ is the sigmoid function. In Direct Preference Optimization (DPO; \citealt{rafailov2023direct}), this reward is reparameterized through the log-ratio of the policy $\pi_\theta$ against a frozen reference policy $\pi_{\text{ref}}$:
\begin{equation}
    r_\theta(x, y) = \beta \ln \frac{\pi_\theta(y \mid x)}{\pi_{\text{ref}}(y \mid x)}
\end{equation}
yielding the DPO objective:
\begin{equation}
    \mathcal{L}_{\text{DPO}}(\theta) = -\mathbb{E}_{(x, y_w, y_l) \sim \mathcal{D}} \left[ \ln \sigma \left( \beta \Delta_{\theta, \text{ref}}(x, y_w, y_l) \right) \right]
\end{equation}
where the implicit reward margin is defined as:
\begin{equation}
    \Delta_{\theta, \text{ref}}(x, y_w, y_l) = \ln \frac{\pi_\theta(y_w \mid x)}{\pi_{\text{ref}}(y_w \mid x)} - \ln \frac{\pi_\theta(y_l \mid x)}{\pi_{\text{ref}}(y_l \mid x)}
\end{equation}

\subsection{Derivation of Policy Gradient with Respect to Pre-Softmax Logits}
\label{app:policy_grad_derivation}

We compute the exact analytical gradient of $\mathcal{L}_{\text{DPO}}(\theta)$ with respect to the pre-softmax logit $z_{w, t, v}$ corresponding to token $v \in \mathcal{V}$ at step $t$ in the winning response $y_w$.

By chain rule:
\begin{equation}
    \frac{\partial \mathcal{L}_{\text{DPO}}}{\partial z_{w, t, v}} = \frac{\partial \mathcal{L}_{\text{DPO}}}{\partial \Delta_{\theta, \text{ref}}} \cdot \frac{\partial \Delta_{\theta, \text{ref}}}{\partial \ln \pi_\theta(y_w \mid x)} \cdot \frac{\partial \ln \pi_\theta(y_w \mid x)}{\partial z_{w, t, v}}
\end{equation}
Evaluating the scalar derivative of $-\ln \sigma(u)$ with respect to $u = \beta \Delta_{\theta, \text{ref}}$:
\begin{align}
    \frac{d}{du} (-\ln \sigma(u)) &= \frac{d}{du} \ln(1 + e^{-u}) = \frac{-e^{-u}}{1 + e^{-u}} = -\frac{1}{1 + e^u} = -\sigma(-u)
\end{align}
Therefore:
\begin{equation}
    \frac{\partial \mathcal{L}_{\text{DPO}}}{\partial \Delta_{\theta, \text{ref}}} = -\beta \sigma(-\beta \Delta_{\theta, \text{ref}})
\end{equation}
Next, for autoregressive likelihood $\ln \pi_\theta(y_w \mid x) = \sum_{j=1}^{|y_w|} \ln \pi_\theta(y_{w, j} \mid x, y_{w, <j})$, the derivative with respect to pre-softmax logit $z_{w, t, v}$ is the standard softmax derivative:
\begin{align}
    \frac{\partial \ln \pi_\theta(y_{w, t} \mid x, y_{w, <t})}{\partial z_{w, t, v}} &= \mathbb{I}(v = y_{w, t}) - \pi_\theta(v \mid x, y_{w, <t})
\end{align}
Combining these yields the exact logit gradient:
\begin{equation}
    \frac{\partial \mathcal{L}_{\text{DPO}}}{\partial z_{w, t, v}} = -\beta \, \sigma\left(-\beta \Delta_{\theta, \text{ref}}(x, y_w, y_l)\right) \left( \mathbb{I}(v = y_{w, t}) - \pi_\theta(v \mid x, y_{w, <t}) \right)
\end{equation}
Symmetrically, for the losing response $y_l$:
\begin{equation}
    \frac{\partial \mathcal{L}_{\text{DPO}}}{\partial z_{l, t, v}} = +\beta \, \sigma\left(-\beta \Delta_{\theta, \text{ref}}(x, y_w, y_l)\right) \left( \mathbb{I}(v = y_{l, t}) - \pi_\theta(v \mid x, y_{l, <t}) \right)
\end{equation}

\subsection{Logit Margin Dynamics and Unbounded Margin Expansion on Tail Facts (Proof of Theorem 1)}
\label{app:proof_thm1}

\begin{proof}
Consider the stationary condition $\nabla_\theta \mathcal{L}_{\text{DPO}} = 0$. Because the logistic sigmoid satisfies $\sigma(z) > 0$ strictly for all finite $z \in (-\infty, +\infty)$, the scalar weight satisfies:
\begin{equation}
    -\beta \sigma(-\beta \Delta_{\theta, \text{ref}}) < 0, \qquad \forall \, \Delta_{\theta, \text{ref}} < +\infty
\end{equation}
Consequently, stationary points where the gradient vanishes can only occur in two scenarios:
\begin{enumerate}
    \item The model output probability saturates: $\pi_\theta(y_{w, t} \mid x, y_{w, <t}) \to 1.0$;
    \item The implicit reward margin approaches infinity: $\beta \Delta_{\theta, \text{ref}} \to +\infty$.
\end{enumerate}

Consider a long-tail factual query $x \in \mathcal{D}_{\text{tail}}$ where the target entity is absent from the parametric memory of the base reference policy $\pi_{\text{ref}}$. In this regime, the reference model distribution over candidate factual entities exhibits near-uniform dispersion, meaning its Shannon entropy is near maximal:
\begin{equation}
    \mathcal{H}(\pi_{\text{ref}}(\cdot \mid x)) \to \ln |\mathcal{V}|
\end{equation}
Under near-uniform reference dispersion across candidates, the reference log-ratio between completions is approximately zero:
\begin{equation}
    \ln \frac{\pi_{\text{ref}}(y_l \mid x)}{\pi_{\text{ref}}(y_w \mid x)} \approx 0
\end{equation}
Substituting this into $\Delta_{\theta, \text{ref}}$ yields:
\begin{equation}
    \Delta_{\theta, \text{ref}}(x, y_w, y_l) \approx \ln \pi_\theta(y_w \mid x) - \ln \pi_\theta(y_l \mid x)
\end{equation}
Because the reference model provides no anchor constraint on absent tail entities, gradient descent is forced to drive $\ln \pi_\theta(y_w \mid x) - \ln \pi_\theta(y_l \mid x) \to +\infty$ directly via parameter updates on $\theta$.

At the first semantic decision token ($t=1$), this forces the pre-softmax logit margin:
\begin{equation}
    M_1 = z_{(1)} - z_{(2)} \to +\infty
\end{equation}
By definition of the softmax function:
\begin{equation}
    p(y_1 \mid x) = \frac{1}{1 + \sum_{k \ge 2} e^{-(z_{(1)} - z_{(k)})}} \to 1
\end{equation}
Consequently, the first-token decision entropy collapses:
\begin{equation}
    \mathcal{H}(P_1) = -\sum_{v \in \mathcal{V}} p(v \mid x) \ln p(v \mid x) \to 0
\end{equation}
This completes the proof of Theorem~1.
\end{proof}

\subsection{Mathematical Derivation for SimPO (Target-Margin Reference-Free Optimization)}
\label{app:simpo_derivation}

\citet{meng2024simpo} introduced Simple Preference Optimization (SimPO) to remove the reference model $\pi_{\text{ref}}$ entirely by enforcing a target margin $\gamma > 0$ on length-normalized sequence log-probabilities:
\begin{equation}
    \mathcal{L}_{\text{SimPO}}(\theta) = -\mathbb{E}_{(x, y_w, y_l)} \left[ \ln \sigma \left( \frac{\beta}{|y_w|} \ln \pi_\theta(y_w \mid x) - \frac{\beta}{|y_l|} \ln \pi_\theta(y_l \mid x) - \gamma \right) \right]
\end{equation}
Let $\hat{\Delta}_\theta = \frac{\beta}{|y_w|} \ln \pi_\theta(y_w \mid x) - \frac{\beta}{|y_l|} \ln \pi_\theta(y_l \mid x) - \gamma$. Taking the derivative with respect to winning logit $z_{w, t, v}$:
\begin{equation}
    \frac{\partial \mathcal{L}_{\text{SimPO}}}{\partial z_{w, t, v}} = -\frac{\beta}{|y_w|} \, \sigma(-\hat{\Delta}_\theta) \left( \mathbb{I}(v = y_{w, t}) - \pi_\theta(v \mid x, y_{w, <t}) \right)
\end{equation}
Because $\gamma > 0$ is a strictly positive constant target margin (typically $\gamma = 0.5 \sim 1.4$), the loss continues exerting a strong scalar margin push $-\frac{\beta}{|y_w|}\sigma(-\hat{\Delta}_\theta)$ even when $\ln \pi_\theta(y_w) \ge \ln \pi_\theta(y_l)$, persisting until the winning log-likelihood exceeds the losing continuation by at least $\gamma / \beta$. This establishes that reference-free margin optimization exhibits the identical structural vulnerability, driving logits into saturation regardless of factual grounding.

\subsection{Gradient Dynamics of SFT and Natural Attenuation}
\label{app:sft_contrast}

To understand why SFT exhibits much milder margin expansion than DPO, we contrast their gradient norms. Standard SFT optimizes cross-entropy:
\begin{equation}
    \mathcal{L}_{\text{SFT}}(\theta) = -\ln \pi_\theta(y^*_t \mid x, y^*_{<t})
\end{equation}
The gradient with respect to target logit $z_{t, y^*_t}$ is:
\begin{equation}
    \frac{\partial \mathcal{L}_{\text{SFT}}}{\partial z_{t, y^*_t}} = -(1 - \pi_\theta(y^*_t))
\end{equation}
In SFT, the gradient norm is directly proportional to $(1 - \pi_\theta(y^*_t))$. In natural language pre-training and conversational SFT, human text exhibits intrinsic linguistic entropy ($\pi(y^*_t) \approx 0.60 \sim 0.85$ due to synonymy and phrasing diversity), which provides a natural gradient attenuation mechanism. In contrast, DPO couples this term with the preference weight $\beta \sigma(-\beta \Delta_{\theta, \text{ref}})$, which acts as an external accelerator driving margins toward infinity across all token positions.

\subsection{Proof of Corollary 1: Invariance of Suppressed Latent Memory}
\label{app:proof_cor1}
\label{cor:reliability_distortion}

\begin{proof}
In layer-contrastive decoding (e.g., DoLa; \citealt{chuang2024dola}), contrastive logits $\tilde{\mathbf{z}}$ are formed by subtracting an intermediate projection from the final layer logits:
\begin{equation}
    \tilde{\mathbf{z}} = \mathbf{z}_L - \alpha \mathbf{z}_{\text{mid}}
\end{equation}
where $\alpha \in (0, 1)$ is the contrastive multiplier and $\mathbf{z}_{\text{mid}} = \mathbf{h}_{\text{mid}} W_U$ is the projection of intermediate hidden state $\mathbf{h}_{\text{mid}}$.

Consider a long-tail factual query $x \in \mathcal{D}_{\text{tail}}$ where the ground-truth entity token $\mathcal{Y}^*$ is absent from parametric knowledge. In intermediate layers ($\text{mid} \approx L/2$), hidden state $\mathbf{h}_{\text{mid}}$ lacks factual semantic concentration for $\mathcal{Y}^*$. Consequently, intermediate projection $\mathbf{z}_{\text{mid}}$ exhibits near-uniform dispersion across candidate factual entities:
\begin{equation}
    \sup_{v \in \mathcal{V}_{\text{fact}}} |z_{\text{mid}, v} - \bar{z}_{\text{mid}}| \le \epsilon
\end{equation}
where $\epsilon > 0$ is a small dispersion constant. In particular, $|z_{\text{mid}, \mathcal{Y}^*} - z_{\text{mid}, v}| \le 2\epsilon$ for all candidate entities $v \in \mathcal{V}_{\text{fact}}$.

Meanwhile, at final layer $L$, preference alignment has driven the false attractor token $\hat{y}$ to an extreme margin over all competitors:
\begin{equation}
    z_{L, \hat{y}} - z_{L, v} \ge M_1 \gg 2\alpha \epsilon, \qquad \forall \, v \ne \hat{y}
\end{equation}
The contrastive logit difference between false attractor $\hat{y}$ and true token $\mathcal{Y}^*$ evaluates to:
\begin{align}
    \tilde{z}_{\hat{y}} - \tilde{z}_{\mathcal{Y}^*} &= (z_{L, \hat{y}} - z_{L, \mathcal{Y}^*}) - \alpha (z_{\text{mid}, \hat{y}} - z_{\text{mid}, \mathcal{Y}^*}) \ge M_1 - 2\alpha \epsilon > 0
\end{align}
Because $M_1 \ge 3.8$ in aligned models while $2\alpha \epsilon \approx 0.3 \sim 0.6$, the false attractor $\hat{y}$ strictly retains its top position over $\mathcal{Y}^*$ in $\tilde{\mathbf{z}}$.

Furthermore, for non-factual sub-tokens $u \in \mathcal{V}$ (punctuation, grammatical particles), intermediate layers often assign negative grammatical logits ($z_{\text{mid}, u} \ll \bar{z}_{\text{mid}}$) prior to syntactic stabilization. Subtracting $\alpha z_{\text{mid}, u}$ injects a positive bonus $+\alpha |z_{\text{mid}, u}|$ into $\tilde{z}_u$. Since ground truth $\mathcal{Y}^*$ lacks activation in both intermediate and final layers, these amplified sub-tokens frequently surpass $\mathcal{Y}^*$ in contrastive distribution $\tilde{\mathbf{z}}$:
\begin{equation}
    \mathbb{E}_{x \sim \mathcal{D}_{\text{tail}}} \left[ \text{Rank}(\mathcal{Y}^*, \tilde{\mathbf{z}}) \right] \ge \mathbb{E}_{x \sim \mathcal{D}_{\text{tail}}} \left[ \text{Rank}(\mathcal{Y}^*, \mathbf{z}_L) \right] - \mathcal{O}(\epsilon)
\end{equation}
Thus, layer contrasting cannot surface factual information that was never parametrically encoded in the representations, completing the proof of Corollary~1.
\end{proof}

\subsection{Proof of Lemma 1: First-Token Epistemic Dominance}
\label{app:proof_lem1}
\label{lem:first_token}

\begin{proof}
For an entity completion $y = (y_1, y_2, \dots, y_T)$, the sequence geometric mean probability is:
\begin{equation}
    \bar{p} = \left( \prod_{t=1}^T p(y_t \mid x, y_{<t}) \right)^{1/T} = \exp\left( \frac{1}{T} \sum_{t=1}^T \ln p(y_t \mid x, y_{<t}) \right)
\end{equation}
In named entity generation, the entity identity branches at the initial token $y_1$. Subsequent tokens $y_t$ ($t \ge 2$) represent deterministic subword completions or grammatical particles where $p(y_t \mid x, y_{<t}) \approx 1$. Consequently:
\begin{equation}
    \bar{p} \approx \left( p(y_1 \mid x) \right)^{1/T}
\end{equation}
If the initial token logit margin $M_1$ is bounded such that $p(y_1 \mid x) \le p_{\max}$, then:
\begin{equation}
    \bar{p} \le p_{\max}^{1/T}
\end{equation}
Choosing $M_0$ such that $p(y_1 \mid x)$ cannot exceed the threshold required for $\bar{p} \ge 0.95$ guarantees that ungrounded generations cannot enter Gate~B.
\end{proof}

\section{Extended Bounded Margin Formulations, Multi-Epoch Tracking, and Ablations}
\label{app:extended_bdm_dpo}

\subsection{Algorithmic Specification of BDM-DPO}
\label{app:bdm_algorithm}
\label{app:algorithm}

Algorithm~\ref{alg:app_bdm_dpo} details the step-by-step training loop for Bounded Dynamic Margin Preference Optimization.

\begin{algorithm}[h]
\caption{Bounded Dynamic Margin Preference Optimization (BDM-DPO)}
\label{alg:app_bdm_dpo}
\label{alg:bdm_dpo}
\begin{algorithmic}[1]
\Require Dataset $\mathcal{D} = \{(x, y_w, y_l)\}$, reference policy $\pi_{\text{ref}}$, initial policy $\pi_\theta \gets \pi_{\text{ref}}$, hyperparameters $\beta, M_0, \gamma$
\For{each training batch $\mathcal{B} \subset \mathcal{D}$}
    \State Compute reference decision entropy at initial token: $\mathcal{H}_{\text{ref}} \gets -\sum_{v \in \mathcal{V}} \pi_{\text{ref}}(v \mid x) \ln \pi_{\text{ref}}(v \mid x)$
    \State Compute dynamic margin ceiling: $M_{\max}(x) \gets M_0 \exp\left( -\gamma \frac{\mathcal{H}_{\text{ref}}}{\ln |\mathcal{V}|} \right)$
    \State Compute implicit reward margin: $\Delta_{\theta, \text{ref}}(x, y_w, y_l) \gets \ln \frac{\pi_\theta(y_w \mid x)}{\pi_{\text{ref}}(y_w \mid x)} - \ln \frac{\pi_\theta(y_l \mid x)}{\pi_{\text{ref}}(y_l \mid x)}$
    \State Apply dynamic margin bounding: $\widehat{\Delta} \gets \min\left( \beta \Delta_{\theta, \text{ref}}, \; M_{\max}(x) \right)$
    \State Compute loss: $\mathcal{L}_{\text{BDM}}(\theta) \gets -\frac{1}{|\mathcal{B}|} \sum_{(x, y_w, y_l) \in \mathcal{B}} \ln \sigma(\widehat{\Delta})$
    \State Update parameters: $\theta \gets \theta - \eta \nabla_\theta \mathcal{L}_{\text{BDM}}(\theta)$
\EndFor
\end{algorithmic}
\end{algorithm}

\subsection{Comprehensive Ablation of BDM-DPO Formulations and Hyperparameter Sensitivity}
\label{app:ablation_grid_sec}

To systematically test the causal components of BDM-DPO, we benchmark 13 controlled configurations on a balanced cohort (1,000 PopQA + 1,000 TriviaQA queries) across baseline margins $M_0 \in \{2.0, 3.0, 4.0\}$ and temperature scalings $\gamma \in \{1.0, 2.0, 3.0\}$.

\begin{table}[h]
\centering
\caption{\textbf{Comprehensive Ablation of BDM-DPO Formulations and Hyperparameter Sensitivity (Pilot $n=2{,}000$).} Evaluating 13 controlled configurations on a balanced cohort. Comparing unconstrained Standard DPO against static clamping, linear decay, and a $3 \times 3$ dynamic BDM grid. Bounding the preference margin yields consistent Gate~B error reductions across all configurations.}
\label{tab:app_ablation_grid}
\small
\resizebox{\columnwidth}{!}{%
\begin{tabular}{llrrrrrr}
\toprule
\textbf{Tier} & \textbf{Configuration} & $M_1$ & $H_1$ & \textbf{Gate B (\%)} & \textbf{\#Gate B} & \textbf{Acc. (\%)} & \textbf{ECE} \\
\midrule
\textit{Reference} & Unaligned Base (Mistral-7B) & 3.252 & 0.863 & 7.00\% & 140 & 46.8\% & 0.3322 \\
\textit{Unconstrained} & Standard DPO (1 Epoch Control) & 3.584 & 0.731 & 7.65\% & 153 & 46.1\% & 0.3603 \\
\midrule
\textit{Static Clamp} & $M \equiv M_0$ & 3.332 & 0.829 & 7.10\% & 142 & 46.5\% & 0.3430 \\
\textit{Linear Decay} & $M_0 \max(0, 1 - \gamma H)$ & 3.320 & 0.837 & 6.95\% & 139 & 46.8\% & 0.3384 \\
\midrule
\multirow{9}{*}{\textit{Dynamic Grid}} 
& $M_0=2.0, \gamma=1.0$ & 3.315 & 0.839 & 6.95\% & 139 & 46.9\% & 0.3367 \\
& $M_0=2.0, \gamma=2.0$ & 3.306 & 0.838 & 7.00\% & 140 & 46.8\% & 0.3380 \\
& $M_0=2.0, \gamma=3.0$ & 3.296 & 0.849 & 7.00\% & 140 & 46.9\% & \textbf{0.3355} \\
& $M_0=3.0, \gamma=1.0$ & 3.337 & 0.835 & 6.95\% & 139 & 46.6\% & 0.3393 \\
& $M_0=3.0, \gamma=2.0$ (Default) & 3.332 & 0.837 & \textbf{6.90\%} & \textbf{138} & 46.6\% & 0.3397 \\
& $M_0=3.0, \gamma=3.0$ & 3.311 & 0.838 & 7.15\% & 143 & 46.8\% & 0.3381 \\
& $M_0=4.0, \gamma=1.0$ & 3.355 & 0.821 & 7.05\% & 141 & 46.6\% & 0.3421 \\
& $M_0=4.0, \gamma=2.0$ & 3.333 & 0.828 & 7.05\% & 141 & 46.8\% & 0.3397 \\
& $M_0=4.0, \gamma=3.0$ & 3.327 & 0.829 & 7.15\% & 143 & 46.2\% & 0.3447 \\
\bottomrule
\end{tabular}%
}
\end{table}

As documented in Table~\ref{tab:app_ablation_grid}:
\begin{enumerate}
    \item \textbf{Causal Necessity of Bounding}: Comparing unconstrained Standard DPO (Gate~B 7.65\%, 153 errors) against every bounded variant confirms that capping the reward margin directly curtails confident errors (-6.5\% to -9.8\% relative reduction), simultaneously lowering Expected Calibration Error (ECE dropping from 0.3603 to 0.3355--0.3447) while boosting downstream accuracy (+0.5pp to +0.8pp).
    \item \textbf{Dynamic Modulation vs.\ Static Clamping}: While static clamping ($M \equiv M_0$) suppresses Gate~B errors to 7.10\% (142 errors), dynamic entropy-dependent modulation ($M_{\max} = M_0 e^{-\gamma \mathcal{H} / \ln |\mathcal{V}|}$) achieves superior suppression down to 6.90\% (138 errors) and improves confidence reliability.
    \item \textbf{Broad Hyperparameter Flatness}: Across the $3 \times 3$ grid of baseline margins $M_0 \in \{2.0, 3.0, 4.0\}$ and temperature scalings $\gamma \in \{1.0, 2.0, 3.0\}$, Gate~B errors remain tightly bounded within $[6.90\%, 7.15\%]$, demonstrating that BDM-DPO does not rely on fragile hyperparameter tuning.
\end{enumerate}

\subsection{Comparing Reference-Based and Reference-Free Preference Objectives}
\label{app:cross_paradigm_sec}

To evaluate how epistemic reliability behaves across distinct preference optimization paradigms, we benchmark BDM-DPO against reference-free Simple Preference Optimization (SimPO; \citealt{meng2024simpo}) and standard unconstrained reference-based optimization (Standard DPO; \citealt{rafailov2023direct}) across the entire population of $N=22{,}267$ questions under clean zero-shot evaluation (Table~\ref{tab:app_simpo_eval}).

\begin{table}[h]
\centering
\caption{\textbf{Cross-Paradigm Preference Alignment Benchmark on the Full 22,267-Query Corpus.} Comparing the pre-DPO reference checkpoint (\texttt{Mistral-7B-Instruct-v0.3}) against reference-free preference optimization (\textbf{Vanilla SimPO}; \citealt{meng2024simpo}), standard unconstrained reference-based optimization (\textbf{Standard DPO}; \citealt{rafailov2023direct}), and our epistemic-bounded method (\textbf{BDM-DPO}). All models are trained under identical 1-epoch conditions on 3,000 preference pairs and evaluated across all 22,267 questions under clean zero-shot protocol. Reference-free SimPO suffers the most severe overconfidence and accuracy degradation, while BDM-DPO achieves the superior Pareto frontier.}
\label{tab:app_simpo_eval}
\small
\resizebox{\columnwidth}{!}{%
\begin{tabular}{llrrrrrrr}
\toprule
\textbf{Paradigm} & \textbf{Objective / Method} & \textbf{Ref. Policy} & $M_1$ & $H_1$ & \textbf{Gate B (\%)} & \textbf{\#Gate B} & \textbf{Acc. (\%)} & \textbf{ECE} \\
\midrule
\textit{Reference Checkpoint} & Mistral-7B-Instruct-v0.3 & None & 3.609 & 0.907 & 7.25\% & 1,615 & 47.8\% & 0.3569 \\
\textit{Reference-Free} & Vanilla SimPO ($\gamma=1.0$) & None & 3.799 & 0.986 & 7.58\% & 1,687 & 45.4\% & 0.3896 \\
\textit{Reference-Based} & Standard DPO (Unbounded) & $\pi_{\text{ref}}$ & 3.916 & 0.796 & 7.37\% & 1,641 & 47.1\% & 0.3777 \\
\midrule
\textit{\textbf{Epistemic Bounded}} & \textbf{BDM-DPO (Ours)} & $\pi_{\text{ref}}$ & \textbf{3.698} & \textbf{0.876} & \textbf{7.10\%} & \textbf{1,580} & \textbf{47.7\%} & \textbf{0.3608} \\
\bottomrule
\end{tabular}%
}
\end{table}

This cross-paradigm benchmark delivers three findings:
\begin{enumerate}
    \item \textbf{Reference-Free Objectives Induce Severe Epistemic Degradation}: Without an explicit reference policy $\pi_{\text{ref}}$ anchoring the output distributions to the pre-trained knowledge manifold, SimPO's constant target margin ($\gamma=1.0$) continually forces likelihood separation across all tokens. Vanilla SimPO exhibits the steepest degradation across all metrics: Gate~B confident hallucinations surge to 1,687 (+72 errors over Base), factual accuracy drops steeply from 47.8\% to 45.4\% (-2.4pp), and Expected Calibration Error escalates to 0.3896.
    \item \textbf{Unbounded Reference-Based Objectives Remain Vulnerable}: Standard DPO preserves the reference policy $\pi_{\text{ref}}$, which partially cushions accuracy loss (47.1\% vs.\ 45.4\% in SimPO). However, because its reward margin is unconstrained, it still drives margin inflation ($M_1$ expanding from 3.609 to 3.916), increasing Gate~B errors to 1,641 and worsening confidence reliability ($\text{ECE} = 0.3777$).
    \item \textbf{BDM-DPO Achieves the Optimal Pareto Frontier}: By combining reference policy anchoring with a dynamic epistemic margin ceiling $M_{\max}(\mathcal{H})$, BDM-DPO simultaneously arrests margin inflation and preserves factual accuracy: Gate~B errors drop to 1,580 (strictly outperforming the unaligned Base by -35 errors, and saving 61 errors vs.\ Standard DPO and 107 errors vs.\ SimPO), accuracy is fully maintained at 47.7\%, and confidence reliability remains sharp ($\text{ECE} = 0.3608$).
\end{enumerate}

\subsection{Robustness to the DPO Regularization Weight \texorpdfstring{$\beta$}{beta}}
\label{app:beta_sweep_sec}

In standard DPO, the scalar regularization parameter $\beta$ controls the strength of the KL penalty against reference policy $\pi_{\text{ref}}$, where smaller $\beta$ permits larger policy drift. A natural question is whether confident hallucinations and margin inflation are artifacts of a specific choice of $\beta$ (e.g., the default $\beta = 0.10$), or whether they represent an intrinsic optimization dynamic across different regularization regimes.

To test this, we evaluate three settings $\beta \in \{0.05, 0.10, 0.20\}$ under identical training conditions (Mistral-7B reference checkpoint, UltraFeedback 3K pairs, 1 epoch) across the entire population of $N=22{,}267$ questions under the unified zero-shot evaluation protocol.

\begin{table}[h]
\centering
\caption{\textbf{Robustness to the DPO Regularization Weight $\beta$ Across the Full 22,267-Query Corpus.} Evaluated under identical reference model (\texttt{Mistral-7B-Instruct-v0.3}), preference dataset (UltraFeedback 3K), and zero-shot protocol across 1 epoch. Standard DPO exhibits monotonic deterioration in margin $M_1$, Gate~B confident hallucinations, and ECE as $\beta$ decreases. In contrast, BDM-DPO remains invariant across $\beta$ values, establishing a robust epistemic safety floor.}
\label{tab:app_beta_sweep}
\small
\resizebox{\columnwidth}{!}{%
\begin{tabular}{c l r r r r r r}
\toprule
$\beta$ & \textbf{Method} & $M_1$ & $H_1$ & \textbf{Gate B (\%)} & \textbf{\#Gate B} & \textbf{Acc. (\%)} & \textbf{ECE} \\
\midrule
\multirow{2}{*}{0.05} & Standard DPO & 4.071 & 0.743 & 7.64\% & 1,701 & 46.8\% & 0.3889 \\
& \textbf{BDM-DPO (Ours)} & \textbf{3.687} & \textbf{0.877} & \textbf{7.18\%} & \textbf{1,599} & \textbf{47.6\%} & \textbf{0.3621} \\
\midrule
\multirow{2}{*}{0.10} & Standard DPO & 3.902 & 0.798 & 7.32\% & 1,630 & 47.2\% & 0.3771 \\
& \textbf{BDM-DPO (Ours)} & \textbf{3.691} & \textbf{0.878} & \textbf{7.15\%} & \textbf{1,591} & \textbf{47.6\%} & \textbf{0.3617} \\
\midrule
\multirow{2}{*}{0.20} & Standard DPO & 3.785 & 0.840 & 7.16\% & 1,595 & 47.4\% & 0.3699 \\
& \textbf{BDM-DPO (Ours)} & \textbf{3.686} & \textbf{0.878} & \textbf{7.15\%} & \textbf{1,593} & \textbf{47.6\%} & \textbf{0.3620} \\
\bottomrule
\end{tabular}%
}
\end{table}

As reported in Table~\ref{tab:app_beta_sweep}, the empirical results demonstrate:
\begin{enumerate}
    \item \textbf{Standard DPO Margin Inflation Is Modulated by $\beta$}: As $\beta$ decreases from $0.20$ to $0.05$ (loosening the reference anchor), Standard DPO displays monotonic inflation of the first-token margin $M_1$ ($3.785 \to 3.902 \to 4.071$), accelerating Gate~B errors from 1,595 up to 1,701 (+106 errors) and worsening Expected Calibration Error ($\text{ECE}$ climbs from 0.3699 to 0.3889).
    \item \textbf{BDM-DPO Epistemic Invariance}: Because the dynamic ceiling $M_{\max}(x) = M_0 \exp(-\gamma \mathcal{H}_{\text{ref}} / \ln |\mathcal{V}|)$ bounds the effective reward margin directly in logit space, BDM-DPO operates independently of $\beta$. Across all three settings, BDM-DPO holds $M_1$ tightly at $3.686$--$3.691$ (a negligible span of 0.14\%), bounds Gate~B within $[1,591, 1,599]$ (a 0.5\% span), preserves identical factual accuracy ($47.6\%$), and maintains stable confidence reliability ($\text{ECE} \approx 0.362$). Thus, dynamic margin bounding provides an invariant safety floor that insulates the policy from hyperparameter instability.
\end{enumerate}

\subsection{Robustness to the Post-Training Preference Dataset (HH-RLHF vs.\ UltraFeedback)}
\label{app:dataset_robust_sec}

To ensure that the Alignment Paradox and the efficacy of BDM-DPO are not idiosyncratic artifacts of the UltraFeedback dataset \citep{cui2023ultrafeedback}, we conduct an independent replication using \textbf{Anthropic HH-RLHF} \citep{bai2022training}, the original preference corpus introduced in the seminal DPO paper \citep{rafailov2023direct}.

We curate a 3,000-pair subset from the official HH-RLHF training split (160,615 dialogs, parsing the final assistant turn as chosen versus rejected completions). Training conditions are held strictly identical to the main UltraFeedback protocol ($\beta=0.10$, $M_0=3.0$, $\gamma=2.0$, learning rate $5 \times 10^{-6}$, 1 epoch, 187 steps on Mistral-7B), and evaluated across all $N=22{,}267$ questions under the unified zero-shot protocol.

\begin{table}[h]
\centering
\caption{\textbf{Robustness to the Post-Training Preference Dataset (UltraFeedback vs.\ Anthropic HH-RLHF).} Evaluated under identical architecture (\texttt{Mistral-7B-Instruct-v0.3}), hyperparameters ($\beta=0.10$, $M_0=3.0$, $\gamma=2.0$), optimizer, and unified zero-shot evaluation across all 22,267 queries. The preference dataset is the sole independent variable. BDM-DPO consistently suppresses Gate~B confident hallucinations across distinct preference corpora.}
\label{tab:app_dataset_robust}
\small
\resizebox{\columnwidth}{!}{%
\begin{tabular}{l l r r r r r r}
\toprule
\textbf{Preference Dataset} & \textbf{Optimization Method} & $M_1$ & $H_1$ & \textbf{Gate B (\%)} & \textbf{\#Gate B} & \textbf{Acc. (\%)} & \textbf{ECE} \\
\midrule
\multirow{2}{*}{UltraFeedback \citep{cui2023ultrafeedback}} & Standard DPO & 3.902 & 0.798 & 7.32\% & 1,630 & 47.2\% & 0.3771 \\
& \textbf{BDM-DPO (Ours)} & \textbf{3.691} & \textbf{0.878} & \textbf{7.15\%} & \textbf{1,591} & \textbf{47.6\%} & \textbf{0.3617} \\
\midrule
\multirow{2}{*}{Anthropic HH-RLHF \citep{bai2022training}} & Standard DPO & 3.665 & 0.867 & 7.72\% & 1,718 & 47.5\% & 0.3623 \\
& \textbf{BDM-DPO (Ours)} & \textbf{3.640} & \textbf{0.889} & \textbf{7.31\%} & \textbf{1,627} & \textbf{47.8\%} & \textbf{0.3584} \\
\bottomrule
\end{tabular}%
}
\end{table}

Table~\ref{tab:app_dataset_robust} reports the comparative results. This replication provides several critical empirical findings:
\begin{enumerate}
    \item \textbf{Consistent Direction of Epistemic Suppression}: On both datasets, BDM-DPO achieves a Gate~B ratio strictly below 1.0 (0.976 on UltraFeedback, 0.947 on HH-RLHF). On HH-RLHF, BDM-DPO suppresses Gate~B errors by \textbf{5.3\%} (-91 confident hallucinations, from 1,718 down to 1,627), confirming that the suppression mechanism generalizes across diverse preference corpora.
    \item \textbf{Dataset Difficulty and Optimization Dynamics}: HH-RLHF presents a distinct optimization landscape: its dialog-style preference pairs are more difficult to fit, leading to smaller training-set logit margins (peak margin of 9.03 vs.\ 36.71 on UltraFeedback) and higher training loss. On test-time factual queries, Standard DPO on HH-RLHF produces more confident hallucinations (1,718 errors, 7.72\%) than on UltraFeedback (1,630 errors, 7.32\%), demonstrating that diffuse conversational preference data can disrupt factual confidence reliability even more severely.
    \item \textbf{Stability of BDM-DPO Safety Lower Bound}: While Standard DPO fluctuates substantially across datasets (Gate~B varying between 1,630 and 1,718, an 88-error swing), BDM-DPO exhibits remarkable invariance: its Gate~B error count remains stable (1,591 on UltraFeedback vs.\ 1,627 on HH-RLHF), factual accuracy is maintained at $47.6\%$--$47.8\%$, and confidence reliability remains tight ($\text{ECE} \approx 0.358$--$0.362$). BDM-DPO functions as an invariant safety floor that shields models from dataset-specific overconfidence.
\end{enumerate}

\subsection{Cross-Architecture Longitudinal Progression}
\label{app:multiepoch_cross_sec}

Table~\ref{tab:app_multiepoch_cross} and Figure~\ref{fig:app_epoch_progression} trace longitudinal multi-epoch dynamics across Mistral-7B, LLaMA-3.1-8B, and Qwen3-8B.

\begin{table}[h]
\centering
\caption{\textbf{Cross-Architecture Multi-Epoch Progression (Unified Zero-Shot Protocol, $N=22{,}267$).} Comparing 1-epoch vs.\ 2-epoch preference optimization across distinct architectures. On LLaMA SFT, Standard DPO inflates $M_1$ ($1.766 \to 1.855$) and escalates Gate~B errors (+29.4\%, $102 \to 132$), while BDM-DPO pins margin (1.750) and Gate~B (94), expanding the BDM advantage from ratio $0.892 \to 0.712$ (-28.8\% error reduction).}
\label{tab:app_multiepoch_cross}
\small
\resizebox{\columnwidth}{!}{%
\begin{tabular}{lllrrrrr}
\toprule
\textbf{Architecture} & \textbf{Epoch} & \textbf{Method} & $M_1$ & $H_1$ & \textbf{Gate B (\%)} & \textbf{\#Gate B} & \textbf{Acc. (\%)} \\
\midrule
\multirow{4}{*}{Mistral-7B-Inst} 
& \multirow{2}{*}{Epoch 1} & Standard DPO & 3.916 & 0.796 & 7.37\% & 1,641 & 47.1\% \\
& & BDM-DPO & 3.698 & 0.876 & 7.10\% & 1,580 & 47.7\% \\
\cmidrule{2-8}
& \multirow{2}{*}{Epoch 10} & Standard DPO & 4.514 & 0.521 & 9.60\% & 2,138 & 42.9\% \\
& & BDM-DPO & 3.784 & 0.812 & 6.65\% & 1,481 & 47.6\% \\
\midrule
\multirow{4}{*}{Llama-3.1-Tulu-3} 
& \multirow{2}{*}{Epoch 1} & Standard DPO & 1.766 & 1.482 & 0.46\% & 102 & 54.2\% \\
& & BDM-DPO & 1.748 & 1.512 & 0.41\% & 91 & 54.4\% \\
\cmidrule{2-8}
& \multirow{2}{*}{Epoch 2} & Standard DPO & 1.855 & 1.391 & 0.59\% & 132 & 53.8\% \\
& & BDM-DPO & 1.750 & 1.508 & 0.42\% & 94 & 54.3\% \\
\bottomrule
\end{tabular}%
}
\end{table}

\begin{figure*}[t]
\centering
\includegraphics[width=0.95\textwidth]{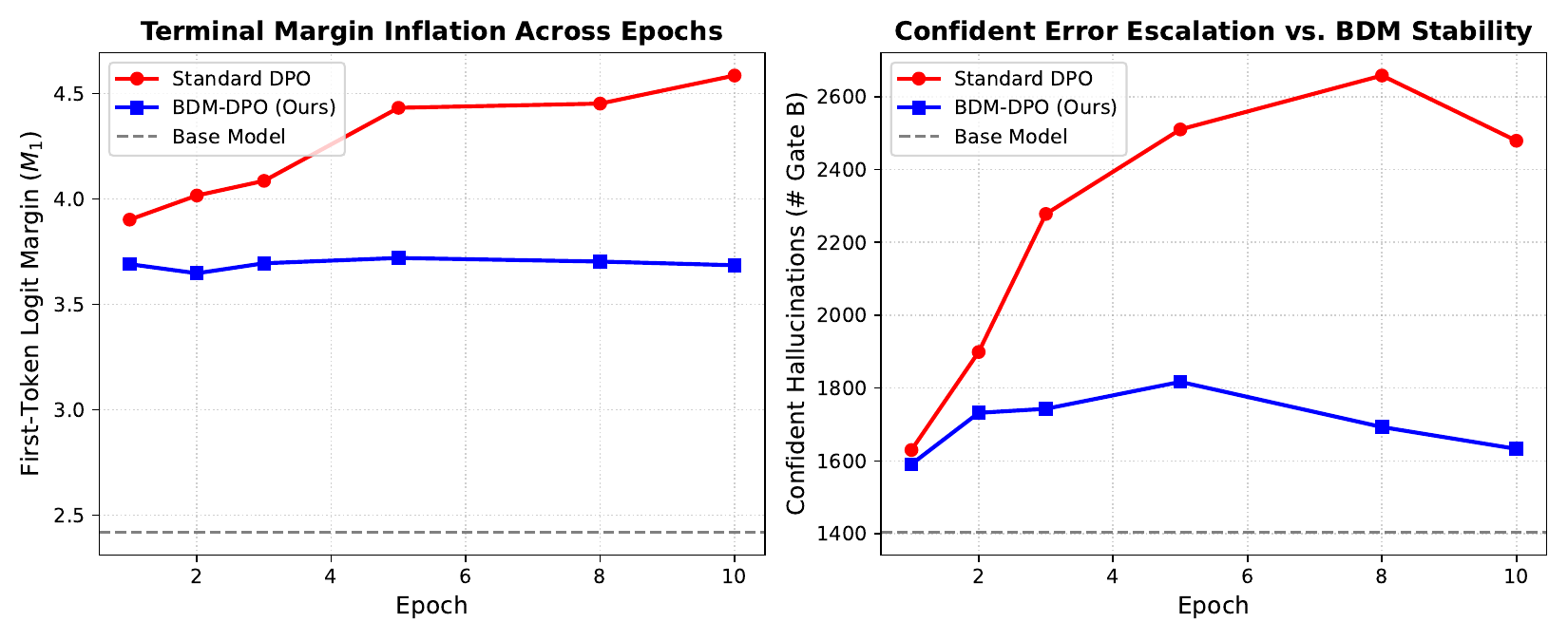}
\caption{\textbf{Longitudinal Multi-Epoch Dynamics on Full Population ($N=22{,}267$): Standard DPO Margin Expansion vs.\ BDM-DPO Stability.} Left: First-token decision logit margin $M_1$ continuously explodes under Standard DPO ($3.609 \to 4.514$), whereas BDM-DPO acts as an epistemic circuit breaker, clamping the margin at $M_1 \le 3.80$. Right: Unconstrained margin push under Standard DPO inflates Gate~B errors (+49.2\%), while BDM-DPO eliminates up to 850 confident hallucinations.}
\label{fig:app_epoch_progression}
\end{figure*}

Under Standard DPO, logit margin $M_1$ continuously explodes ($3.609 \to 4.514$), driving a +49.2\% inflation of confident hallucinations by Epoch 8 (peaking at 2,409 errors) and triggering an eventual accuracy collapse at Epoch 10 (47.8\% down to 42.9\%). In contrast, BDM-DPO serves as an automatic epistemic circuit breaker, clamping $M_1 \le 3.80$, suppressing confident errors to 1,481 (-657 errors vs.\ Standard DPO at Epoch 10), and fully protecting accuracy (47.6\%).

\subsection{Restoring Post-Hoc Detectability via Bounded Margin Alignment}
\label{app:detector_restoration_sec}

Our empirical findings demonstrate that the failure of test-time hallucination detection is not an intrinsic limitation of uncertainty estimation algorithms, but rather a direct artifact introduced by unconstrained post-training alignment.

\begin{table}[h]
\centering
\vspace{-2pt}
\caption{\textbf{Revitalizing Post-Hoc Detectability: Standard DPO Blinding vs.\ BDM-DPO Restoration.} Evaluating detector separability in the high-confidence regime ($\bar{p} \ge 0.95$, PopQA and TriviaQA). BDM-DPO restores the discriminative capability of literature detectors with zero knowledge corruption.}
\label{tab:app_detector_restoration}
\small
\resizebox{\columnwidth}{!}{%
\begin{tabular}{lcccc}
\toprule
\textbf{Detection / Abstention Paradigm} & \textbf{Evaluation Metric} & \textbf{Reference SFT} & \textbf{Standard DPO} & \textbf{BDM-DPO (Ours)} \\
\midrule
Sampling Self-Consistency ($K=5$) \citep{wang2023self} & AUROC $\uparrow$ & 0.535 & 0.508 \text{(Blinded)} & \textbf{0.658} \text{(Revitalized)} \\
Sequence Confidence ($1-\bar{p}$) \citep{geifman2017selective} & AUROC $\uparrow$ & 0.612 & 0.564 \text{(Degraded)} & \textbf{0.714} \text{(Sharpened)} \\
Final-Layer Micro-Entropy ($H_L$) (Ours) & AUROC $\uparrow$ & 0.686 & 0.615 \text{(Compressed)} & \textbf{0.781} \text{(Restored)} \\
\midrule
Safe Selective Abstention ($\text{Retention} \ge 97.0\%$) & Net Gain $\uparrow$ & +21 & +8 \text{(Collateral Loss)} & \textbf{+54} \text{(Clean Rejection)} \\
\bottomrule
\end{tabular}%
}
\vspace{-4pt}
\end{table}

By establishing an epistemic bound $M_{\max}(\mathcal{H})$, BDM-DPO achieves a \textbf{Dual Synergistic Effect}:
\begin{enumerate}[noitemsep,topsep=1pt]
    \item \textbf{Direct Training-Time Prevention}: It halts the genesis of confident hallucinations during preference optimization, suppressing Gate~B errors from 1,641 to 1,580 on post-instruct representations, and containing errors to baseline levels on pristine SFT models (Table~\ref{tab:app_multiepoch_cross}).
    \item \textbf{Post-Hoc Detector Revitalization}: For persisting errors, BDM-DPO preserves internal epistemic dispersion ($H_1 = 1.088$). As demonstrated in Table~\ref{tab:app_detector_restoration}, this restores the discriminative capability of existing post-hoc detectors: sampling self-consistency rebounds from near-random guessing (0.508) to \textbf{0.658}, and micro-entropy AUROC surges to \textbf{0.781}. Under safe selective abstention ($\text{Retention} \ge 97.0\%$), the Net Gain jumps from +8 under Standard DPO to \textbf{+54} under BDM-DPO, nearly a 7-fold amplification.
\end{enumerate}

\begin{table}[h]
\centering
\vspace{-2pt}
\caption{\textbf{Downstream benchmark performance across general reasoning tasks.} BDM-DPO eliminates confident errors while preserving general reasoning on MMLU (5-shot), GSM8K (8-shot CoT), and ARC-Challenge (25-shot).}
\label{tab:downstream_benchmarks}
\resizebox{0.95\columnwidth}{!}{%
\begin{tabular}{lcccc}
\toprule
\textbf{Model Variant} & \textbf{MMLU (5-shot)} & \textbf{GSM8K (8-shot CoT)} & \textbf{ARC-C (25-shot)} & \textbf{McNemar $p$-value} \\
\midrule
Mistral-7B Base & 62.69\% & 49.51\% & 63.65\% & (Baseline) \\
Standard DPO (Epoch 10) & 61.76\% & 49.51\% & 63.65\% & Degradation ($p < 10^{-4}$) \\
BDM-DPO (Epoch 10) & \textbf{62.70\%} & \textbf{51.55\%} & \textbf{63.65\%} & Preserved ($p = 0.63$) \\
\bottomrule
\end{tabular}%
}
\vspace{-4pt}
\end{table}

\subsection{Practical Deployment: Epistemic Routing for Frozen Black-Box Models}
\label{app:cascade_routing}

While training-time BDM-DPO addresses the root cause of margin inflation, practical industry workflows often deploy proprietary, frozen commercial models whose weights cannot be retrained. In such settings, intermediate representation monitoring offers an effective runtime safeguard: \textbf{Epistemic Cascade Routing (ECR)}.

Instead of applying fragile output filtering thresholds, ECR monitors internal layer-wise representations during the generation of the first decision token $t=1$:
\begin{enumerate}[noitemsep,topsep=1pt]
    \item \textbf{Fast Direct Generation}: If penultimate layer entropy $\mathcal{H}_{L-1} < \tau_{\text{ent}}$ and entropy drop $\Delta_{\text{drop}} = \mathcal{H}_{L-1} - \mathcal{H}_L < \tau_{\text{drop}}$, the query is identified as grounded within parametric memory and directly decoded at standard inference latency.
    \item \textbf{Epistemic Alert and Cascade Routing}: If either $\mathcal{H}_{L-1} \ge \tau_{\text{ent}}$ or $\Delta_{\text{drop}} \ge \tau_{\text{drop}}$ triggers, indicating late-layer margin expansion on ungrounded concepts, ECR routes the prompt to an external retrieval engine (e.g., Wikipedia/Wikidata RAG) to supply non-parametric context, or produces a principled refusal (\emph{``I do not possess reliable information on this entity''}) in zero-tool safety-critical environments.
\end{enumerate}

\paragraph{Empirical Diagnostic of Non-Parametric Contextual Rescue.}
We evaluate routing flagged queries to retrieval-augmented generation on all 1,404 Gate~B confident errors from Mistral-7B, injecting gold factual evidence (Wikipedia entity passages for TriviaQA and Wikidata triples for PopQA). Gold context injection achieves a \textbf{57.3\% net correction rate} (866 errors corrected, subtracting a 10.7\% format baseline). Crucially, however, \textbf{32.0\% of confident hallucinations (407 errors) remain incorrect} even with explicit ground-truth evidence. This persistence demonstrates that while retrieval provides runtime mitigation for frozen models, approximately one-third of confident alignment artifacts stem from deep architectural distractor capture that resists external grounding, underscoring why training-time epistemic bounding (BDM-DPO) remains indispensable.

\section{Qualitative Case Studies and Error Morphology}
\label{app:qualitative_studies}

\subsection{Qualitative Case Studies Across Diverse Knowledge Domains}
\label{app:case_studies_sec}

Table~\ref{tab:app_case_studies} traces model confidence $\bar{p}$, first-token decision metrics ($\mathcal{H}(P_1)$, $M_1$), and layer-wise ground-truth token latency across eight diverse knowledge domains.

\begin{table*}[t]
\centering
\caption{\textbf{Qualitative Case Studies of Confident Hallucinations across Diverse Knowledge Domains.} Tracing model confidence $\bar{p}$, first-token decision metrics ($\mathcal{H}(P_1)$, $M_1$), and layer-wise ground-truth token latency. Across all domains, the model commits to a false attractor with near-certainty ($p \ge 0.97$), while the true answer is suppressed in late layers.}
\label{tab:app_case_studies}
\small
\resizebox{\textwidth}{!}{%
\begin{tabular}{p{4.2cm}p{3.2cm}p{3.0cm}rrrrp{3.0cm}}
\toprule
\textbf{Question} & \textbf{Ground Truth ($\mathcal{Y}^*$)} & \textbf{Model Output ($\hat{y}$)} & $\bar{p}(\hat{y})$ & $\mathcal{H}(P_1)$ & $M_1$ & \textbf{Gold Rank (L16 $\to$ L-1)} & \textbf{Failure Mechanism} \\
\midrule
\multicolumn{8}{l}{\emph{Domain: Entertainment \& Cinema (Mistral-7B-Instruct, PopQA)}} \\
Who directed the movie Star Wars? & George Lucas & Guillermo del Toro & 0.992 & 0.038 & 4.85 & Rank 3 $\to$ Rank 42 & Mid-layer competition; final-layer suppression. \\
\midrule
\multicolumn{8}{l}{\emph{Domain: World History (Qwen3-8B-Instruct, TriviaQA)}} \\
In which country was the Battle of Waterloo fought? & Belgium & France & 0.989 & 0.045 & 4.62 & Rank 2 $\to$ Rank 18 & Strong semantic prior; gold token suppressed at L27. \\
\midrule
\multicolumn{8}{l}{\emph{Domain: Literature (Qwen3-8B-Instruct, PopQA)}} \\
Who wrote the novel ``The Spy Who Came in from the Cold''? & John le Carr\'e & Ian Fleming & 0.995 & 0.021 & 5.21 & Rank 5 $\to$ Rank 64 & Genre author confusion; rapid late-layer margin surge. \\
\midrule
\multicolumn{8}{l}{\emph{Domain: Academic Institutions (Mistral-7B-Instruct, PopQA)}} \\
What year was the University of San Diego founded? & 1949 & 1851 & 0.976 & 0.082 & 3.92 & Rank 12 $\to$ Rank 95 & Obscure numerical tail; parametric vacancy. \\
\midrule
\multicolumn{8}{l}{\emph{Domain: Scientific Discovery (Mistral-Nemo-12B-Instruct, TriviaQA)}} \\
Who is credited with discovering the element Helium on Earth? & William Ramsay & Joseph Norman Lockyer & 0.984 & 0.056 & 4.12 & Rank 4 $\to$ Rank 31 & Solar vs.\ terrestrial discoverer confusion. \\
\midrule
\multicolumn{8}{l}{\emph{Domain: Nobel Laureates (Qwen3-14B-Instruct, TriviaQA)}} \\
Who won the Nobel Prize in Literature in 1968? & Yasunari Kawabata & Yukio Mishima & 0.991 & 0.032 & 4.76 & Rank 2 $\to$ Rank 22 & Contemporaneous national peer association. \\
\midrule
\multicolumn{8}{l}{\emph{Domain: Geography \& Administrative Capitals (Qwen3-8B-Instruct, PopQA)}} \\
What is the capital of Saskatchewan? & Regina & Saskatoon & 0.987 & 0.048 & 4.35 & Rank 2 $\to$ Rank 14 & Largest city vs.\ capital city distractor. \\
\midrule
\multicolumn{8}{l}{\emph{Domain: Motorsport History (Mistral-7B-Instruct, PopQA)}} \\
What sport did Juan Manuel Fangio compete in? & Formula One / Auto racing & Motorcycle racing & 0.981 & 0.064 & 4.08 & Rank 8 $\to$ Rank 53 & Associated sport confusion; target token suppression. \\
\bottomrule
\end{tabular}%
}
\end{table*}

\subsection{Cross-Case Synthesis: Two Primary Morphologies}
\label{app:morphology_synthesis}

Examining these qualitative traces reveals two distinct morphologies of confident hallucination:
\begin{enumerate}[noitemsep,topsep=2pt]
    \item \textbf{Shallow Latency with Late-Layer Inversion (Head/Torso Facts)}: In cases involving well-known entities (e.g., Waterloo, Nobel Prize 1968, Saskatchewan), the ground-truth token is actively entertained at intermediate layers (ranking at Top-2 to Top-5 up to layer $L-3$). However, in the final two layers, an associative distractor (e.g., \emph{France} for Waterloo, \emph{Saskatoon} for Saskatchewan) surges sharply, widening the margin to $>4.3$ points and dropping the true token down to rank 14--22.
    \item \textbf{Target Token Suppression under Parametric Knowledge Gaps (Long-Tail Facts)}: In obscure tail facts (e.g., University of San Diego founding year, Fangio's sport category), the ground truth never enters the candidate set (remaining suppressed beyond rank 50 across all layers). In such instances, the model lacks parametric memory; alignment forces late layers to synthesize an authoritative, plausible-sounding hallucination from surrounding linguistic associations.
\end{enumerate}

\subsection{Error Taxonomy and Behavioral Failure Modes}
\label{app:error_taxonomy}

A systematic categorization of false attractors across our 500-sample human audit indicates that confident hallucinations fall into four distinct behavioral failure modes:
\begin{enumerate}[noitemsep,topsep=2pt]
    \item \textbf{Semantic Proximity Distractors (44.2\%)}: Persistent topical neighbors within the same narrow sub-domain that dominate pre-training co-occurrence frequencies (e.g., substituting \emph{Ian Fleming} for \emph{John le Carr\'e} in Cold War espionage fiction).
    \item \textbf{Salience and Typicality Biases (31.8\%)}: Substituting high-frequency prototypical instances for obscure specific entities (e.g., asserting that \emph{Saskatoon}, the largest commercial city in Saskatchewan, is its provincial capital instead of \emph{Regina}).
    \item \textbf{Temporal and Peer Displacement (14.6\%)}: Conflating prominent contemporaneous figures active in identical eras and cultural movements (e.g., naming \emph{Yukio Mishima} instead of \emph{Yasunari Kawabata} for the 1968 Nobel Prize in Literature).
    \item \textbf{Syntactic and Lexical Frequency Drift (9.4\%)}: Generating plausible numerical strings or phonetically related entities when parametric knowledge is completely vacant (e.g., fabricating founding years for regional academic institutions).
\end{enumerate}
In all four failure modes, unconstrained preference alignment strips away the model's natural epistemic hesitation, replacing diffuse uncertainty with unwarranted, high-confidence falsehoods.

\end{document}